\documentclass[twoside]{article}

\usepackage[accepted]{aistats2020}
\usepackage{amsthm}
\usepackage{amsmath}
\usepackage{amssymb}
\usepackage{xcolor,url,paralist}
\usepackage{graphicx}
\usepackage{algorithm,algorithmic}
\usepackage{subfigure}
\usepackage[round]{natbib}

\def\od#1#2{\frac{d#1}{d#2}}
\def\pd#1#2{\frac{\partial #1}{\partial #2}}
\def\fd#1#2{\frac{\delta #1}{\delta #2}}
\def\tpd#1#2{\partial #1/\partial #2}

\def\tfd#1#2{\delta #1/\delta #2}

\def\dzero#1#2{\left.\od{}{#1} #2 \right|_{#1=0}}

\def\grad{\mathop{\mathrm{grad}}}

\def\parentheses#1{\!\left(#1\right)}

\def\tr{\mathop{\mathrm{tr}}\nolimits}

\def\diag{\operatorname{diag}}

\def\R{\mathbb{R}}

\def\defeq{\mathrel{\mathop:}=}
\def\eqdef{=\mathrel{\mathop:}}

\def\ip#1#2{{\left\langle#1,#2\right\rangle}}

\newcommand{\id}{\operatorname{id}}

\def\d{\mathbf{d}}
\def\ins#1{{\bf i}_{#1}}

\def\FL{\mathbb{F}L}

\newcommand\ad{\operatorname{ad}}

\def\SO{\mathsf{SO}}

\def\so{\mathfrak{so}}

\makeatletter
\DeclareFontFamily{OMX}{MnSymbolE}{}
\DeclareSymbolFont{MnLargeSymbols}{OMX}{MnSymbolE}{m}{n}
\SetSymbolFont{MnLargeSymbols}{bold}{OMX}{MnSymbolE}{b}{n}
\DeclareFontShape{OMX}{MnSymbolE}{m}{n}{
    <-6>  MnSymbolE5
   <6-7>  MnSymbolE6
   <7-8>  MnSymbolE7
   <8-9>  MnSymbolE8
   <9-10> MnSymbolE9
  <10-12> MnSymbolE10
  <12->   MnSymbolE12
}{}
\DeclareFontShape{OMX}{MnSymbolE}{b}{n}{
    <-6>  MnSymbolE-Bold5
   <6-7>  MnSymbolE-Bold6
   <7-8>  MnSymbolE-Bold7
   <8-9>  MnSymbolE-Bold8
   <9-10> MnSymbolE-Bold9
  <10-12> MnSymbolE-Bold10
  <12->   MnSymbolE-Bold12
}{}

\let\llangle\@undefined
\let\rrangle\@undefined
\DeclareMathDelimiter{\llangle}{\mathopen}%
                     {MnLargeSymbols}{'164}{MnLargeSymbols}{'164}
\DeclareMathDelimiter{\rrangle}{\mathclose}%
                     {MnLargeSymbols}{'171}{MnLargeSymbols}{'171}
\makeatother

\def\metric#1#2{{\left\llangle#1,#2\right\rrangle}}

\newcommand{\argmin}{\mathop{\mathrm{argmin}}}
  
\newenvironment{tbmatrix}{\left[\begin{smallmatrix}}{\end{smallmatrix}\right]}

\theoremstyle{definition}

\newtheorem{Remark}{Remark}[section]

\theoremstyle{plain}
\newtheorem{Proposition}{Proposition}[section]
\newtheorem{Theorem}{Theorem}[section]
\newtheorem{Corollary}{Corollary}[section]

\begin{document}

%

%

\twocolumn[

\aistatstitle{Variational Optimization on Lie Groups, with Examples of Leading (Generalized) Eigenvalue Problems}

\aistatsauthor{ Molei Tao \And Tomoki Ohsawa }

\aistatsaddress{ Georgia Institute of Technology \And University of Texas at Dallas } 
]

\begin{abstract}
  The article considers smooth optimization of functions on Lie groups. By generalizing NAG variational principle in vector space \citep{wibisonoe7351} to Lie groups, continuous Lie-NAG dynamics which are guaranteed to converge to local optimum are obtained. They correspond to momentum versions of gradient flow on Lie groups. A particular case of $\SO(n)$ is then studied in details, with objective functions corresponding to leading Generalized EigenValue problems: the Lie-NAG dynamics are first made explicit in coordinates, and then discretized in structure preserving fashions, resulting in optimization algorithms with faithful energy behavior (due to conformal symplecticity) and exactly remaining on the Lie group. Stochastic gradient versions are also investigated. Numerical experiments on both synthetic data and practical problem (LDA for MNIST) demonstrate the effectiveness of the proposed methods as optimization algorithms (\emph{not} as a classification method).
\end{abstract}

\section{Introduction}
The algorithmic task of optimization is important in data sciences and other fields. For differentiable objective functions, 1st-order optimization algorithms have been popular choices especially for high dimensional problems, largely due to their scalability, generality, and robustness. A celebrated class of them is based on Nesterov Accelerated Gradient descent (NAG; see e.g., \citep{nesterov1983method,nesterov2018lectures}), also known as a major way to add momentum to Gradient Descent (GD). NAGs enjoy great properties such as quadratic decay of error (instead of GD's linear decay) for convex but not strongly convex objective functions. In addition, the introduction of momentum in NAG softens the dependence of convergence rate on the condition number of the problem. Since high dimensional problems often correspond to larger condition numbers, it is conventional wisdom that adding momentum to gradient descent makes it scale better with high dimensional problems (e.g., \cite{ruder2016overview}, and \cite{cheng2018underdamped} for rigorous results on related problems).

In particular, at least two versions of NAG have been widely used, referred to as NAG-SC and NAG-C for instance in \cite{shi2018understanding}. While their original versions are iterative methods in discrete time, their continuum limits (as the step size goes to zero) have also been studied: for example, \cite{su2014differential} thoroughly investigates these limits as ODEs, and \cite{wibisonoe7351} establishes a corresponding variational principle (along with other  generalizations). Further developments exist; for instance, \cite{shi2018understanding} discusses how to better approximate the original NAGs by high-resolution NAG-ODEs when step size is small but not infinitesimal, and was followed up by \cite{WaTa19Hessian}. Note, however, that no variational principle has been provided yet for the high-resolution NAG-ODEs, to the best of our knowledge.

Although the aforementioned discussions on NAG are in the context of finite dimensional vector space, a variational principle can allow it to be intrinsically generalized to manifolds. Such generalizations are meaningful, because objective functions may not always be a function on vector space, and abundant applications require optimization with respect to parameters in curved spaces. The first part of this article  generalizes continuous NAG dynamics to Lie groups, which are differentiable manifolds that are also groups. Special orthogonal group $\SO(n)$, which contains $n$-by-$n$ real orthogonal matrices with determinant 1, is a classical Lie group, and its optimization is not only relevant to data sciences (see e.g., Sec.\ref{sec_SOn} and Appendix) but also to physical sciences. Some more examples include symplectic groups, spin groups, and unitary groups, all of which play important roles in contemporary physics (e.g., \cite{sattinger2013lie}); for instance, optimization on unitary groups found applications in quantum control (e.g., \cite{glaser1998unitary}), quantum information (e.g., \cite{kitaev2000parallelization}), MIMO communication systems (e.g., \cite{abrudan2009conjugate}), and NMR spectroscopy (e.g., \cite{sorensen1989polarization}).

Variational principles on Lie groups (or more precisely, on the tangent bundle of Lie groups, for introducing velocity) provide a Lagrangian point of view for mechanical systems on Lie groups, and have been extensively studied in geometric mechanics (e.g., \cite{marsden2013introduction,hoscst2009}). Nevertheless, the application of geometric mechanics to NAG-type optimization in this article is new. The second part of this article will discretize the resulting NAG-dynamics on Lie groups, which lead to actual optimization algorithms. These algorithms are also new, although they can certainly be embedded as part of the profound existing field of geometric numerical integration (e.g., the classic monograph of \cite{hairer2006geometric}).

It is also important to mention that optimization on manifolds is already a field so rich that only an incomplete list of references can be provided, e.g., \cite{gabay1982minimizing, smith1994optimization, edelman1998geometry, absil2009optimization, patterson2013stochastic, zhang2016first, zhang2016riemannian, liu2017accelerated, boumal2018global, ma2019there, zhang2018towards, liu2018accelerated}. However, a specialization in Lie group will still be helpful, because the additional group structure (joined efforts with NAG) improves the optimization; for instance, a well known reduction is to, under symmetry, pull the velocity at any location on the Lie group back the tangent space at the identity (known as the Lie algebra).

We also note that NAG (either in vector space or on Lie group) is not restricted to convex optimization. In fact, the proposed methods will be demonstrated on an example of (leading) (Generalized) EigenValues (GEV) problems, which is known to be nonconvex (e.g., \cite{chi2019nonconvex} and its references therein).

GEV is a classical linear algebra problem behind tasks including Linear Discriminant Analysis (see Sec.\ref{sec_LDA} and Appendix) and Canonical Correlation Analysis (e.g., \cite{barnett1987origins}). Due to its importance, numerous GEV algorithms exist (see e.g., \cite{saad2011numerical}), some iterative (e.g., variants of power method) and some direct (e.g., Lanczos-based methods). And we choose GEV as an example to demonstrate our method applied to Lie group $\SO(n)$.

Meanwhile, another line of approaches has also been popular, especially for data sciences problems, often referred to as Oja flow \citep{oja_simplified_1982}, Sanger's rule \citep{sanger_optimal_1989}, and Generalized Hebbian Algorithm \citep{gorrell_generalized_nodate}. While initially proposed for the leading eigenvalue problem, they extend to the leading GEV problem (e.g., \cite{chen_constrained_2019}). For a simple notation, we follow \cite{chen_constrained_2019} and denote them by `GHA'. GHA is based on a matrix-valued ODE, whose long time solution converges to a solution of GEV; more details are reviewed in Appendix. Since the GHA ODE has to be discretized and numerically solved, GHA in practice is still an iterative method, but it is a special one: because of its ODE nature, GHA adapts well to a stochastic generalization of GEV, in which one only has access to noisy/incomplete realizations of the actual matrix (see Sec.\ref{sec_method_SG} for more details), and hence remains popular in machine learning. The proposed methods will also be based on ODEs and suitable to stochastic problems, and thus they will be compared with GHA (Sec.\ref{sec_SG_experiments}). Worth mentioning is, GEV is still being actively investigated; besides \cite{chen_constrained_2019}, recent progress include, for instance, \cite{ge2016efficient,allen2017doubly,arora2017stochastic}. While the main contribution of this article is the momentum-based general Lie group optimization methodology (\textbf{not} GEV algorithms), the derived GEV algorithms are complementary to states-of-arts, because the proposed methods are indifferent to eigengap unlike \cite{ge2016efficient}, and no direct access or inversion of the constraining matrix as different from \cite{allen2017doubly,arora2017stochastic}; however, our method can be made stochastic but not `doubly-stochastic'.

%
%
%
%
%
%

This article is organized as follows. Sec.\ref{sec_generalThry} derives the continuous Lie-group optimization dynamics based on the NAG variational principle. Sec.\ref{sec_SOn_conti} describes, at the continuous level, the case when the Lie group is $\SO(n)$, including the (full) eigenvalue problem and the leading GEV problem; both NAG dynamics and GD (no momentum) are discussed. Sec.\ref{sec_SOn_discrete} then describes discretized algorithms, and Sec.\ref{sec_method_SG} extends them to stochastic problems. Sec.\ref{sec_experiments} provides numerical evidence of the efficacy of our methods, with demonstrations on both synthetic and real data.

\paragraph{Quick user guide:} For GEV, a family of NAG dynamics were obtained. The simplest ones are
\vspace{-2pt}
\begin{flalign}
\text{\underline{Lie-GD}:}	&&
  \dot{R} = R ([R^{T} A R, \mathcal{E}])	\qquad\qquad
\label{eq_LieGD_conti}
\end{flalign}
\vspace{-20pt}
Initial condition has to satisfy: $R(0)^T B R(0)=I$.

\begin{flalign}
\text{\underline{Lie-NAG}:}	&& 
  \dot{R} = R \xi,
  ~
    \dot{\xi}
  = -\gamma(t) \xi
  + [R^{T} A R, \mathcal{E}]
\label{eq_LieNAG_conti}
\end{flalign}
where $\mathcal{E} :=
  \begin{tbmatrix}
    I_{l} & 0 \\
    0 & 0
  \end{tbmatrix}_{n\times n}$.
Initial conditions have to satisfy: $R(0)^T B R(0)=I$ and $\xi(0)^T=-\xi(0)$.

Constant $\gamma$ and $\gamma(t)=3/t$ respectively correspond to Lie-NAG-SC and Lie-NAG-C. If it is affordable to tune the constant $\gamma$ value, our general recommendation is Lie-NAG-SC. Its associated optimization algorithm is Algm.\ref{alg_LieNAG}, and Algm.\ref{alg_LieGD} is also provided for Lie-GD.

\section{Variational Optimization on Lie Group: the General Theory}
\label{sec_generalThry}
\subsection{Gradient Flow}
Our focus is optimization problems on Lie groups:
Let $\mathsf{G}$ be a compact Lie group, $f\colon \mathsf{G} \to \R$ be a smooth function, and consider the optimization problem
\begin{equation*}
  \min_{g \in \mathsf{G}}f(g).
\end{equation*}
We may define the gradient flow for this problem as follows:
Let $T\mathsf{G}$ and $T^{*}\mathsf{G}$ be the tangent and cotangent bundles of $\mathsf{G}$, $e \in \mathsf{G}$ be the identity, and $\mathfrak{g} \defeq T_{e}\mathsf{G}$ be the Lie algebra of $\mathsf{G}$.
Suppose that $\mathfrak{g}$ is equipped with an inner product $\metric{\xi}{\eta} \defeq \ip{\mathbb{I}\xi}{\eta}$ with an isomorphism $\mathbb{I}\colon \mathfrak{g} \to \mathfrak{g}^{*}; \xi \mapsto \mathbb{I}(\xi)$ where $\mathfrak{g}^{*}$ is the dual of the Lie algebra $\mathfrak{g}$, and $\ip{\,\cdot\,}{\,\cdot\,}$ stands for the natural dual pairing.
One can naturally extend this metric to a left-invariant metric on $\mathsf{G}$ by defining, $\forall g \in \mathsf{G}$ and $\forall v, w \in T_{g}\mathsf{G}$, $\metric{v}{w} \defeq \metric{T_{g}\mathsf{L}_{g^{-1}}(v)}{T_{g}\mathsf{L}_{g^{-1}}(w)}$, where $\mathsf{L}_{g}\colon \mathsf{G} \to \mathsf{G}; h \mapsto g h$ is the left translation by $g \in \mathsf{G}$ and $T_{h}\mathsf{L}_{g}\colon T_{h}\mathsf{G} \to T_{g h}\mathsf{G}$ is its tangent map.

Now, we define the gradient vector field $\grad f$ on $\mathsf{G}$ as follows:
For any $g \in \mathsf{G}$ and any $\dot{g} \in T_{g}\mathsf{G}$,
\begin{equation*}
  \metric{ (\grad f)(g) }{ \dot{g} } = \ip{ \d{f}(g) }{ \dot{g} }
  \quad
  \forall g \in \mathsf{G}
  \quad
  \forall \dot{g} \in T_{g}\mathsf{G},
\end{equation*}
where $\d$ stands for the exterior differential.
This gives
\begin{gather*}
  (\grad f)(g) = T_{e}\mathsf{L}_{g} \circ \mathbb{I}^{-1} \circ T_{e}^{*}\mathsf{L}_{g}(\d{f}(g)),
\end{gather*}
where $T_{e}^{*}\mathsf{L}_{g}$ is the dual of $T_{e}\mathsf{L}_{g}$, i.e., $\forall\alpha_{g} \in T_{g}^{*}\mathsf{G}$ and $\forall\xi \in \mathfrak{g}$, $\ip{T_{e}^{*}\mathsf{L}_{g}(\alpha_{g})}{\xi} = \ip{\alpha_{g}}{T_{e}\mathsf{L}_{g}(\xi)}$.
Hence the gradient descent equation is given by
\begin{equation}
  \label{eq:gradient_flow}
  \dot{g} = -(\grad f)(g) = -T_{e}\mathsf{L}_{g} \circ \mathbb{I}^{-1} \circ T_{e}^{*}\mathsf{L}_{g}(\d{f}(g)).
\end{equation}

\subsection{Adding Momentum: the Variational Optimization}
Our work provides a natural extension of variational optimization of \citet{wibisonoe7351} to Lie groups making use of the geometric formulation of the Euler--Lagrange equation on Lie groups.
Specifically, let us define the Lagrangian $L \colon T\mathsf{G} \times \R \to \R$ as follows:
\begin{equation}
  \label{eq:L-TG}
  L(g, \dot{g}, t) \defeq r(t) \parentheses{ \frac{1}{2}\metric{\dot{g}}{\dot{g}} - f(g) },
\end{equation}
where $r\colon \R \to \R_{>0}$ is a smooth positive-valued function.
Instead of working with the tangent bundle $T\mathsf{G}$ directly, it is more convenient to use the left-trivialization of $T\mathsf{G}$, i.e., we may identify $T\mathsf{G}$ with $\mathsf{G} \times \mathfrak{g}$ via the map $\mathsf{G} \times \mathfrak{g} \to T\mathsf{G}; (g, \xi) \mapsto \parentheses{ g, T_{e}\mathsf{L}_{g}(\xi) }$.
Under this identification, we have the Lagrangian $L \colon \mathsf{G} \times \mathfrak{g} \times \R \to \R$ defined as
\begin{equation}
  \label{eq:L}
  L(g, \xi, t) \defeq r(t)\parentheses{ \frac{1}{2}\ip{\mathbb{I}(\xi)}{\xi} - f(g) }.
\end{equation}
The Euler--Lagrange equation for this Lagrangian is (see, e.g., \citet[Section~7.3]{hoscst2009} and also \citet{marsden2013introduction})
\begin{equation*}
  \label{eq:E-L_in_L}
  \od{}{t}\parentheses{ \fd{L}{\xi} } = \ad_{\xi}^{*}\fd{L}{\xi} + T_{e}^{*}\mathsf{L}_{g}(\d_{g}L),
\end{equation*}
along with 
\begin{equation}
\label{eq:reconstruction}
	\dot{g} = T_{e}\mathsf{L}_{g}(\xi) \eqdef g \xi,
\end{equation}
where $\ad^{*}$ is the coadjoint operator; $\tfd{L}{\xi} \in \mathfrak{g}^{*}$ is defined so that, for any $\delta\xi \in \mathfrak{g}$,
\begin{equation*}
  \ip{ \fd{L}{\xi} }{ \delta\xi }
  = \dzero{s}{L(g, \xi + s\delta\xi, t)};
\end{equation*}
also note that $\d_{g}L$ stands for the exterior differential of $g \mapsto L(g,\xi,t)$.
Using the above expression~\eqref{eq:L} of the Lagrangian, we obtain
\begin{equation}
  \label{eq:E-L}
  \od{}{t}\mathbb{I}(\xi)
  = -\gamma(t) \mathbb{I}(\xi)
  + \ad_{\xi}^{*}\mathbb{I}(\xi)
  - T_{e}^{*}\mathsf{L}_{g}(\d{f}(g)),
\end{equation}
where we defined $\gamma(t) \defeq r'(t)/r(t)$.

\paragraph{Choices of $\gamma$.} We will mainly consider $\gamma(t)=\gamma$ (constant) and $\gamma(t)=3/t$, derived from $r=\exp(\gamma t)$ and $r=t^3$. In vector space, these two choices respectively correspond to, as termed for instance in \cite{shi2018understanding}, NAG-SC and NAG-C, which are the continuum limits of two classical versions of Nesterov's Accelerated Gradient methods \citep{nesterov1983method,nesterov2018lectures}.

\paragraph{Lyapunov function.} Let $t \mapsto (g(t),\xi(t))$ be a solution of eq.~\eqref{eq:E-L}.
Assuming that $g_{0}$ is an isolated local minimum of $f$, we can show that the dynamics starting in a neighborhood of $g_{0}$ converges to $g_{0}$ as follows.
Define the ``energy'' function $E\colon \mathsf{G} \times \mathfrak{g} \to \R$ as
\begin{equation}
  \label{eq:E}
  E(g,\xi) \defeq \frac{1}{2}\metric{\xi}{\xi} + f(g)
  = \frac{1}{2}\ip{\mathbb{I}(\xi)}{\xi} + f(g).
\end{equation}
This gives a Lyapunov function.
In fact, there exists a neighborhood $U$ of $(g_{0},0)$ such that $E(g,\xi) \ge f(g) > f(g_{0})$ for any $(g,\xi) \in U\backslash\{(g_{0},0)\}$.
Moreover, we have $\od{}{t} E(g(t),\xi(t)) = -\gamma \metric{ \xi }{ \xi } \le 0$, where the equality implies $\xi = 0$, for which \eqref{eq:E-L} gives $\d{f}(g) = 0$, which locally gives $g = g_{0}$.

\section{The Example of $\SO(n)$ and Its Application to Leading GEV}
\label{sec_SOn}
\subsection{The Continuous Formulations}
\label{sec_SOn_conti}

\subsubsection{The Symmetric Eigenvalue Problem}
Let $A$ be a real symmetric $n \times n$ matrix, and define, as in \cite{brockett1989least,MaMa2002},
\begin{equation*}
  f \colon \SO(n) \to \R;
  \qquad
  f(R) \defeq \tr(R^{T}A R \mathcal{N}),
\end{equation*}
where $\mathcal{N} \defeq \diag(1, 2, \dots, n)$.
We equip the Lie algebra $\so(n)$ with the inner product $\metric{\xi}{\eta} \defeq \tr(\xi^{T}\eta)$.
Then we may identify $\so(n)^{*}$ with $\so(n)$ via this inner product.
Then the ``force'' term in \eqref{eq:E-L} is given by
$T_{I}^{*}\mathsf{L}_{R}(\d{f}(R)) = [R^{T} A R, \mathcal{N}]$.
Since $\ad_{\xi}^{*}\mu = [\mu,\xi]$ for any $\xi \in \so(n)$ and $\mu \in \so(n)^{*} \cong \so(n)$, \eqref{eq:E-L} becomes
\begin{equation}
  \label{eq_qgrejvboqob3hbr5}
  \dot{R} = R \xi,
  \
    \dot{\xi}
  = -\gamma \xi
  + \mathbb{I}^{-1} \parentheses{
    [\mathbb{I}(\xi), \xi]
    - [R^{T} A R, \mathcal{N}]
  },
\end{equation}
whereas the gradient descent equation~\eqref{eq:gradient_flow} gives
\begin{equation}
  \label{eq_gqhuregorq8h09hb}
  \dot{R} = - R \mathbb{I}^{-1}([R^{T} A R, \mathcal{N}]).
\end{equation}

\begin{Remark}[Rigorous results v.s. intuitive addition of momentum]
	The above dynamics work for any positive definite isomorphism $\mathbb{I}\colon \mathfrak{g} \to \mathfrak{g}^{*}$. For simplicity, we will use $\mathbb{I}=\text{id}$ (where $\mathfrak{g}^{*}$ is identified with $\mathfrak{g}$) in implementations in this article. In this case, the $[\mathbb{I}(\xi), \xi]$ term and the $\mathbb{I}^{-1}$ operation vanish, and the momentum version \eqref{eq_qgrejvboqob3hbr5} is heuristically obtainable from \eqref{eq_gqhuregorq8h09hb} just like how momentum was added to gradient flow in vector spaces. Otherwise, they create additional nontrivial nonlinearities that account for the curved space.
\end{Remark}

\begin{Remark}[Relation to double-bracket]
  When $\mathbb{I} = \id$, the gradient flow~\eqref{eq_gqhuregorq8h09hb} becomes $\dot{R} = - R([R^{T} A R, \mathcal{N}])$. By setting $M(t) \defeq R(t)^{T} A R(t)$, we recover the double-bracket equation $\dot{M} = -[M,[M,\mathcal{N}]]$ of ~\citet{Br1991} (see also \citet{BlBrRa1992}).
  Note that there is a sign difference from \cite{Br1991} because \citeauthor{Br1991}'s is gradient \textit{ascent}.
\end{Remark}

\begin{Remark}[Generality]
	The proposed methods, Lie-NAG \eqref{eq_qgrejvboqob3hbr5} and Lie-GD \eqref{eq_gqhuregorq8h09hb}, are indifferent to the absolute location of $A$'s eigenvalues, because they are invariant to the shift $A \mapsto A+\lambda I$. To see this, note $[R^T A R, \mathcal{N}] \mapsto [R^T (A+\lambda I) R,\mathcal{N}]=[R^T A R, \mathcal{N}]+\lambda[R^T R,\mathcal{N}]=[R^T A R, \mathcal{N}]+\lambda[I,\mathcal{N}]=[R^T A R, \mathcal{N}]$. Therefore, the proposed methods work the same no matter whether $A$ is positive/negative-definite. In the generalized eigenvalue setting (see future Sec.\ref{sec_GEV}), the same reasoning and invariance hold for $L^{-T} A L^{-1} \mapsto L^{-T} A L^{-1}+\lambda I$ where $L^T L=B$.
\label{rmk_shiftInvariance}
\end{Remark}

\subsubsection{The Leading $l$ Eigenvalue Problem}
Let $A$ be a real symmetric $n \times n$ matrix. Since finding the smallest $l$ eigenvalues of $A$ is the same as finding the largest $l$ eigenvalues of $-A$, define
\begin{equation}
  f \colon \SO(n) \to \R;
  \qquad
  f(R) \defeq -\tr(E^{T}R^{T}A R E),
  \label{eq_GEVminimization}
\end{equation}
where $E \defeq \begin{tbmatrix}
  I_{l} \\
  0
\end{tbmatrix}$ is $n \times l$ where $I_{l}$ is the $l \times l$ identity matrix and $0$ is the $(n-l) \times l$ zero matrix.

The cost function is almost the same as the previous case except that $\mathcal{N}$ is now replaced by
\begin{equation*}
  \mathcal{E} \defeq E E^{T} =  
  \begin{tbmatrix}
    I_{l} & 0 \\
    0 & 0
  \end{tbmatrix}.
\end{equation*}
So we have $T_{I}^{*}\mathsf{L}_{R}(\d{f}(R)) = -[R^{T} A R, \mathcal{E}]$.

\subsubsection{The Leading $l$ Generalized Eigenvalues}
\label{sec_GEV}
Consider the leading $l$ Generalized EigenValues problem (GEV): given $n$-by-$n$ symmetric $A$ and $n$-by-$n$ positive definite $B$, we seek an optimizer of
\begin{equation}
  \label{eq_GEV}
  \max_{V\in \mathbb{R}^{n\times l}} \text{tr}(V^T A V)
  \quad\text{s.t.}\quad V^T B V = I_{l\times l} .
\end{equation}
It can be seen, by Cholesky decomposition $B=L^T L$ and a Lie group isomorphism $X\mapsto LX$, that
\begin{Proposition}
	$G=\{X|X\in\mathbb{R}^{n\times n},X^T BX=I\}$ is a Lie group. Its identity is $L^{-1}$, and its multiplication is not the usual matrix multiplication but $X_1\cdot X_2=X_1L X_2$.
\end{Proposition}
Therefore, in theory, GEV can be solved by padding $V$ into $X$ and then following our general approach \eqref{eq:E-L}.

The point of this section is to make this solution explicit, and more importantly, to show $L$ is never explicitly needed, which leads to computational efficiency. In fact, the same NAG dynamics
\begin{equation}
  \dot{R} = R \xi,	\qquad
  \dot{\xi}
  = -\gamma(t) \xi
    + [R^{T} A R, \mathcal{E}]
  \label{eq_avikuh31o87th4o81ho36515}
\end{equation}
with initial conditions satisfying
\[
	R(0)^T B R(0)=I, \quad \xi(0)^T=-\xi(0)
\]
will solve \eqref{eq_GEV} upon projecting the first $l$ columns of $R$ into $V$.

Note the only difference from the previous two sections is the initial condition on $R$. In addition, although positive definite $B$ is needed for the group isomorphism, it is only a sufficient (not necessary) condition for NAG \eqref{eq_avikuh31o87th4o81ho36515} to work.

A rigorous justification of why \eqref{eq_avikuh31o87th4o81ho36515} works for not only EV but also GEV can be found in Appendix, where one will also find the proof of a quick sanity check:
\begin{Theorem}
	Under \eqref{eq_avikuh31o87th4o81ho36515} and consistent initial condition, $R(t)^T B R(t)=I$ and $\xi(t)^T=-\xi(t)$ for all $t$.
	\label{eq_NAGstaysOnLieGroup}
\end{Theorem}
The objective function itself does not decrease monotonically in NAG, because it acts as potential energy, which exchanges with kinetic energy, but the total energy decreases (eq.\ref{eq:E}).

On the other hand, if one considers Lie-GD, which can be shown to generalize to GEV also by only modifying the initial condition (given by \eqref{eq_LieGD_conti}), then not only does $R(t)$ stay on the Lie group $G$ (see Appendix), but also is the objective function $\tr [-(R^T(t)A R(t) \mathcal{E})]$ monotone (by construction).

\subsection{The Discrete Algorithms}
\label{sec_SOn_discrete}
Define Cayley transformation\footnote{It is the same as Pade(1,1) approximation.} as $\text{Cayley}(\xi):=(I-\xi/2)^{-1} (I+\xi/2)$. It will be useful as a 2nd-order structure-preserving approximation of matrix exp, the latter of which is computationally too expensive. More precisely, $\exp(h\xi) = \text{Cayley}(h \xi) +\mathcal{O}(h^3)$.

\vspace{-1ex}
\paragraph{Lie-GD.} We adopt a 1st-order (in $h$) explicit discretization of the dynamics $\dot{R} = R ([R^{T} A R, \mathcal{E}])$:
\vspace{-1ex}
\begin{algorithm}[H]
\caption{A 1st-order Lie-GD for leading GEV}
\begin{algorithmic}[1]
\STATE Initialize with some $R_0$ satisfying $R_0^T B R_0=I$.
\FOR{$i=0,\cdots,$TotalSteps-1}
	\STATE $f_i \leftarrow R_i^T A R_i \mathcal{E} - \mathcal{E}R_i^T A R_i$.
	\STATE $R_{i+1} \leftarrow R_i \text{Cayley}(h f_i)$
\ENDFOR
\STATE Output $R_{\text{TotalSteps}}$ as $\argmin f$ in \eqref{eq_GEVminimization}.
\end{algorithmic}
\label{alg_LieGD}
\end{algorithm}
\vspace{-2.5ex}
Note Algm.\ref{alg_LieGD} is more accurate than forward Euler discretization despite that both are 1st-order. This is because all $R_i$'s it produces will remain on the Lie group (i.e., $R_i^T B R_i=I$; see Thm.\ref{thm_LieGroupPreservation_discrete} in Appendix).

\vspace{-1ex}
\paragraph{Lie-NAG.} We present a 2nd-order (in $h$) explicit discretization of the dynamics $\dot{R} = R \xi, ~ \dot{\xi} = -\gamma(t) \xi + [R^{T} A R, \mathcal{E}]$. Unlike the Lie-GD case, the discretization was achieved by the powerful machinery of operator splitting, and can be easily generalized to arbitrarily high-order (e.g., \cite{mclachlan2002splitting,tao2016explicit}), provided that Cayley transformation was replaced by a higher-order Lie-group-preserving approximation of matrix exponential.

More precisely, denote by $\phi^h$ the exact $h$-time flow of the NAG dynamics, and by $\phi_1^h$ and $\phi_2^h$ some $p$-th order approximations of the $h$-time flows of $\dot{R} = R \xi, ~ \dot{\xi} = 0$ and $\dot{R} = 0, ~ \dot{\xi} = -\gamma(t) \xi + [R^{T} A R, \mathcal{E}]$. Note even though $\phi$ is unavailable, the latter systems are analytically solvable, so if $\exp(\xi h)$ is exactly computed, $\phi_1$ and $\phi_2$ can be made exact. Even if they are just $p$-th order approximations ($p\geq 2$), operator splitting yields $\phi^h =\phi_2^{h/2} \circ \phi_1^h \circ \phi_2^{h/2} + \mathcal{O}(h^3)$. Other ways of composing $\phi_1$,$\phi_2$ can lead to higher order methods (Appendix describes some 4th-order options), with maximum order capped by $p$. For simpler coding, $\dot{\xi} = -\gamma(t) \xi + [R^{T} A R, \mathcal{E}]$ can be further split into $\dot{\xi} = -\gamma(t) \xi$ and $\dot{\xi}=[R^{T} A R, \mathcal{E}]$, and Algm.\ref{alg_LieNAG} is based on $\phi_3^{h/2} \circ \phi_2^{h/2} \circ \phi_1^{h} \circ \phi_2^{h/2} \circ \phi_3^{h/2}$:
\vspace{-1ex}
\begin{algorithm}[H]
\caption{A 2nd-order Lie-NAG for leading GEV}
\begin{algorithmic}[1]
\STATE Initialize with some $R_0$ and $\xi_0$ satisfying $R_0^T B R_0=I$ and $\xi_0^T=-\xi_0$.
\FOR{$i=0,\cdots,$TotalSteps-1}
	\STATE $\xi_{i'} \leftarrow \xi_i + h/2 (R_i^T A R_i \mathcal{E} - \mathcal{E}R_i^T A R_i)$.
	\STATE $\xi_{i'} \leftarrow \begin{cases} 
							\exp(-\gamma h/2) \xi_{i'}, & \quad\text{for NAG-SC} \\ 
							((ih)^3/((i+1/2)h)^3) \xi_{i'},   & \quad\text{for NAG-C}
						   \end{cases}$.
	\STATE $R_{i+1} \leftarrow R_i \text{Cayley}(h \xi_{i'})$.
	\STATE $\xi_{i'} \leftarrow \begin{cases} 
							\exp(-\gamma h/2) \xi_{i'}, & \text{NAG-SC} \\ 
							(((i+1/2)h)^3/((i+1)h)^3) \xi_{i'},   & \text{NAG-C}
						   \end{cases}$.
	\STATE $\xi_{i+1} \leftarrow \xi_{i'} + h/2 (R_{i+1}^T A R_{i+1} \mathcal{E} - \mathcal{E}R_{i+1}^T A R_{i+1})$.
\ENDFOR
\STATE Output $R_{\text{TotalSteps}}$ as $\argmin f$ in \eqref{eq_GEVminimization}.
\end{algorithmic}
\label{alg_LieNAG}
\end{algorithm}
\vspace{-3ex}

Also by Thm.\ref{thm_LieGroupPreservation_discrete}, all $R_i$'s remain on the Lie group if arithmetics have infinite machine precision.

In addition, Algm.\ref{alg_LieNAG} is conformal symplectic (see Appendix), which is indicative of favorable accuracy in long time energy behavior. To prove so, note both $\phi_1$ and $\phi_3$ as exact Hamiltonian flows preserve the canonical symplectic form, and two substeps of $\phi_2$ as linear maps discount it by a multiplicative factor of $r(t_i)/r(t_{i+1})$. This exactly agrees with the continuous theory in Appendix.

\subsection{Generalization to Stochastic Problems}
\label{sec_method_SG}
\paragraph{Setup:} now let us consider a Stochastic Gradient (SG) setup, where one may not have full access to $A$ but only a finite collection of its noisy realizations. More precisely, given one realization of i.i.d.~random matrices $A_1,\cdots,A_K$, the goal is to compute the leading (generalized) eigenvalues of $A=\frac{1}{K}\sum_{k=1}^K A_k$ based on $A_k$'s without explicitly using $A$.
\vspace{-2ex}
\paragraph{Implementation:} following the classical stochastic gradient approach, we simply replace $A$ in each algorithm by $A_\kappa$, where $\kappa$ is a uniform random variable on $[K]$, independently drawn at each timestep. 

\begin{Remark}
	Like \cite{ge2016efficient} and unlike \cite{chen_constrained_2019}, the proposed methods do not allow $B$ to be a stochastic approximation. Only $A$ can be stochastic. On the other hand, unlike both \cite{ge2016efficient} and \cite{chen_constrained_2019}, we do not require a direct access to $B$, and all information about $B$ is reflected in the initial condition $R(0)$.
\end{Remark}
\vspace{-2ex}
\paragraph{Intuition:} we now make heuristic arguments to gain insights about the performance of the method.

First, based on the common approximation of stochastic gradient as batch gradient plus Gaussian noise (see e.g., \cite{li2019stochastic} for some state-of-art quantifications of the accuracy of this approximation), assume $A_\kappa=A+\sigma H$ where $H$ is a symmetric Gaussian matrix (assumed as $H=\Xi+\Xi^T$ where $\Xi$ is an $n$-by-$n$ matrix with i.i.d.~standard normal elements), i.i.d.~at each step. Then the gradient $[R^T A_\kappa R,\mathcal{E}]$ is, in distribution and conditioned on $R$, equal to $[R^T A R,\mathcal{E}] + 2 \sigma \Xi$. This is because $[R^T A_\kappa R,\mathcal{E}]$ is Gaussian and its mean is $[R^T A R,\mathcal{E}]$ and covariance is $\sigma^2 \text{covar}[ [R^T H R,\mathcal{E}]|R]$, which can be computed to be $4 \sigma^2 I$, independent of $R$ as long as $R^T R=I$ and $\mathcal{E}$ is a degenerate identity. Therefore, at least in the case of $\mathbb{I}=\text{id}$, the Lie-NAG SG dynamics can be understood through
\begin{equation}
  \dot{R} = R \xi,	\qquad
  \dot{\xi}
  = -\gamma(t) \xi
    + [R^{T} A R, \mathcal{E}] + 2 \hat{\sigma} E,
  \label{eq_NAG_SG_SDE}
\end{equation}
where $E$ is a skew-symmetric white-noise, i.e., $E_{ij}$ with $i<j$ being i.i.d.~white noise, $E_{ji}=-E_{ij}$, and $E_{ii}=0$, and $\hat{\sigma}=\sigma$ in this continuous setting.

Worth mentioning is, once one uses a numerical discretization, namely
\begin{align*}
	\xi_{i+1} &= \xi_i - h \gamma(t_i)\xi_i + h [R_i^T A_{\kappa_i} R_i, \mathcal{E}] + o(h), \\
	&\stackrel{D}{=} \xi_i - h \gamma(t_i)\xi_i + h [R_i^T A R_i, \mathcal{E}] + h 2\sigma E_i + o(h)
\end{align*}
then since $\kappa$ does not randomize infinitely frequently, the effective noise amplitude $\hat{\sigma}$ gets scaled as
\begin{equation}
	\hat{\sigma}=\sqrt{h}\sigma+o(\sqrt{h}),
	\label{eq_effectiveNoiseForSG}
\end{equation}
because a 1st-order discretization of \eqref{eq_NAG_SG_SDE} should have its $\xi$ component being
\[
	\xi_{i+1} = \xi_i - h \gamma(t_i)\xi_i + h [R_i^T A R_i, \mathcal{E}] + \sqrt{h} 2 \hat{\sigma} E_i + o(h)
\]
due to stochastic calculus. This leads to $
  h 2 \sigma E_i = \sqrt{h} 2 \hat{\sigma} E_i + o(h)$, and hence \eqref{eq_effectiveNoiseForSG}.

Secondly, recall an analogous vector space setting, in which one considers
\[
	\dot{q}=p,	\qquad	\dot{p}=-\gamma p - \nabla V(q) + \hat{\sigma} e
\]
where $e$ is standard vectorial white-noise. It is well known that under reasonable assumptions (e.g., \cite{pavliotis2014stochastic}) this diffusion process admits, and converges weakly to an invariant distribution of $Z^{-1}\exp(-H(q,p)/kT)dq dp$, where $H=\|p\|^2/2+V(q)$ is the Hamiltonian, $Z$ is some normalization constant, and $kT=\hat{\sigma}^2/(2\gamma)$ is the temperature (with unit).

It is easy to see that for the purpose of optimization, the temperature should be small. If one uses vanishing stepsizes, since $\hat{\sigma}^2=h \sigma^2$, $kT \rightarrow 0$, and stochastic optimization can be guaranteed to work (more details in \cite{robbins1951stochastic}). If $h$ is small but not infinitesimal, $q$ (or $R$) is still concentrated near the optimum value(s) with high probability.

Now recall Lie-NAG-SC uses constant $\gamma$; Lie-NAG-C, on the contrary, uses $\gamma(t)=3/t$. This means Lie-NAG-SC equipped with SG converges to some invariant distribution at temperature $h\sigma^2/(2\gamma)$, but Lie-NAG-C-SG's `temperature' $kT=h \sigma^2/(6/t)$ grows unbounded with $t$ for constant $h$; i.e., constant stepsize Lie-NAG-C-SG doesn't converge even in a weak sense.

This is another reason that our general recommendation is Lie-NAG-SC over Lie-NAG-C. On the other hand, there are multiple possibilities to correct the non-convergence of Lie-NAG-C: (i) appropriately vanishing $h$ can lead to recovery of an invariant distribution, but to obtain a fixed accuracy one would need more steps; (ii) one can add a correction to the dynamics \citep{WaTa19Hessian}; (iii) modify $\gamma(t)$.

\paragraph{Corrected dissipation coefficient:} this article experimented with option (iii) with
\begin{equation}
	\gamma=3/t + c t,	\quad\text{where }c\text{ is a small constant;}
	\label{eq_correctedDissipationSchedule}
\end{equation}
see Sec.\ref{sec_SG_experiments}. This choice corresponds to $r(t)=\exp(c t^2/2)t^3$ in the variational formulation. Formally, it leads to 0 temperature, but in practice early stopping is needed because any finite $h$ cannot properly numerical-integrate the dynamics when $\gamma$ becomes sufficiently large.

The reason for choosing the specific linear form of the correction $+ct$ is in Appendix.

\section{Experiments}
\label{sec_experiments}

\subsection{Leading Eigenvalue Problems}
\subsubsection{Bounded Spectrum}
\label{sec_EV_bounded}
We first test the proposed methods on a synthetic problem: finding the $l$ largest eigenvalues of $A=(\Xi+\Xi^T)/2/\sqrt{n}$, where $\Xi$ is a sample of an $n$-by-$n$ matrix with i.i.d. standard normal elements. The scaling of $1/\sqrt{n}$ ensures\footnote{For more precise statement and justification, see random matrix theory for Gaussian Orthogonal Ensemble (GOE), or more generally Wigner matrix \cite{wigner1958distribution}} the leading eigenvalues are bounded by a constant independent of $n$; for an unbounded case, see the next example.

\begin{figure}[ht]
\vspace{-3ex}
\hspace{-20pt}
\includegraphics[width=1.16\columnwidth]{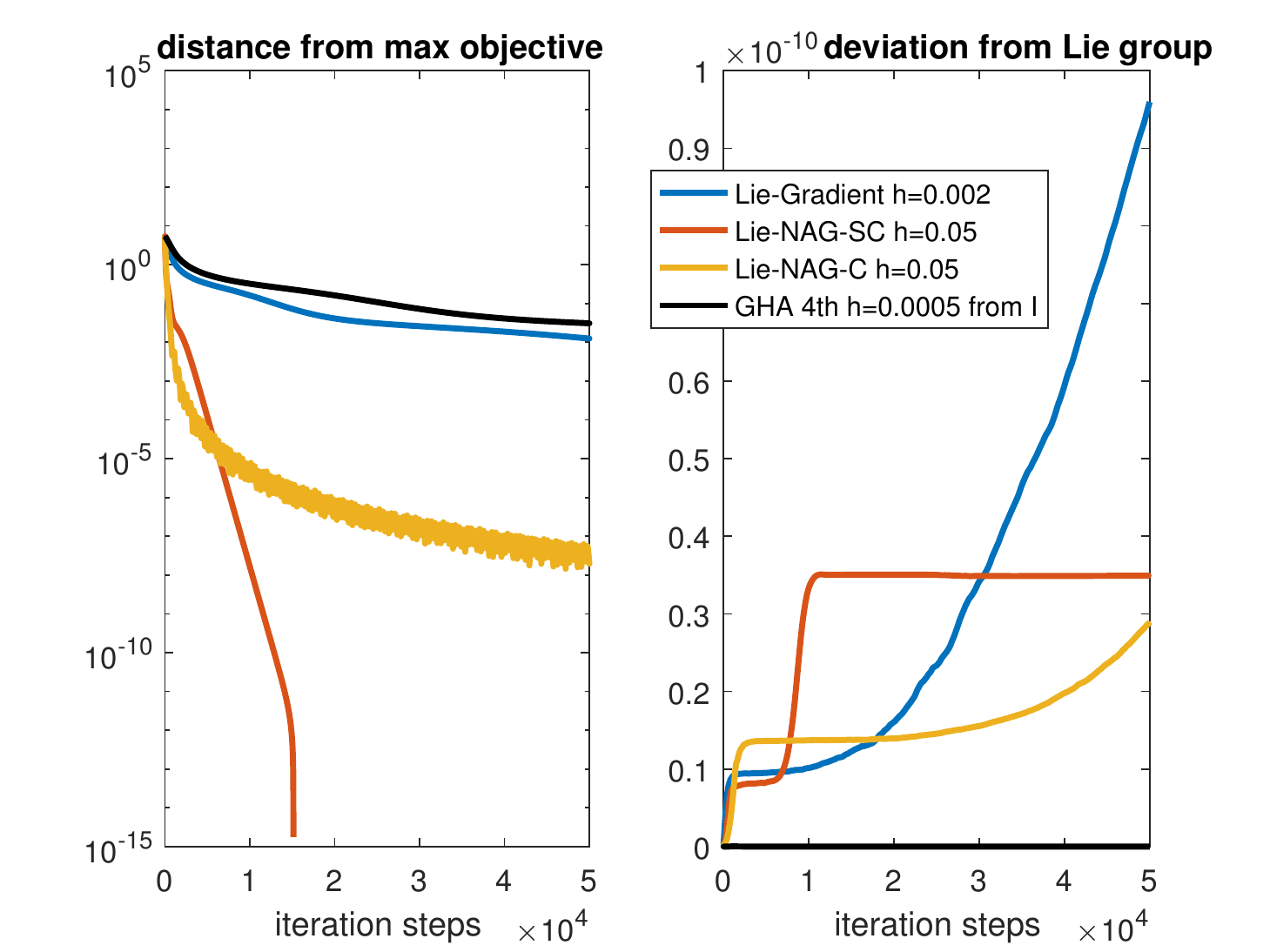}
\hspace{-20pt}
\vspace{-4ex}
\caption{Performances of proposed Lie-GD, Lie-NAG-C and Lie-NAG-SC, compared with GHA, for computing the leading $l=2$ eigenvalues of scaled GOE. All algorithms use step sizes tuned to minimize error in $5\times 10^4$ iterations (although the proposed methods do not need much tuning), and identity initial condition. GHA was based on Runge-Kutta-4 integration of $\dot{Q}=(I-QQ^T)A Q$ for accuracy, and an Euler integration did not result in any notable error reduction. NAG-SC uses friction coefficient untuned $\gamma=1$. The deviations of Lie-NAGs and Lie-GD from the Lie group are machine/platform (MATLAB) precision artifacts.}
\label{fig_largestKevalues_differentA_1b_n500_new}
\vspace{-3.5ex}
\end{figure}

Fig.~\ref{fig_largestKevalues_differentA_1b_n500_new} shows results for a generic sample of 500-dimensional $A$. The proposed Lie-NAG's, i.e. variational methods with momentum, converge significantly faster than the popular GHA. This advantage is even more significant in higher dimensions (see Fig.~\ref{fig_largestKevalues_differentA_1b_n2000} in Appendix). Note Fig.~\ref{fig_largestKevalues_differentA_1b_n500_new} plots accuracy as a function of the number of iterations, and readers interested in accuracy as a function of wallclock are referred to Fig.~\ref{fig_largestKevalues_differentA_1b_n500_new_wallClock} (note wallclock count is platform dependent and therefore the latter illustration is only qualitative but not quantitative, thus placed in the Appendix). In any case, for this problem at least, if low-moderate accuracy is desired, Lie-NAG-C is the most efficient among tested methods; if high accuracy is desired instead, Lie-NAG-SC is the optimal choice.

Note the fact that $A$ has both positive and negative eigenvalues should not impair the credibility of this demonstration. This is because one can shift $A$ to make it positive definite or negative definite, and the convergences will be precisely the same. See Rmk.\ref{rmk_shiftInvariance}.

\subsubsection{Unbounded Spectrum}
Now consider computing the leading eigenvalues of $A=-\Xi \Xi^T/2$ ($\Xi$ similarly defined as in Sec.\ref{sec_EV_bounded}). This is equivalent to finding the $l$ smallest eigenvalues of $\Xi \Xi^T/2$. Doing so is relevant, for instance, in graph theory, where the 2nd smallest eigenvalue of graph Laplacian is the algebraic connectivity of the graph \citep{fiedler1973algebraic, von2007tutorial}.

\begin{figure}[ht]
\vspace{-2.5ex}
\hspace{-20pt}
\includegraphics[width=1.16\columnwidth]{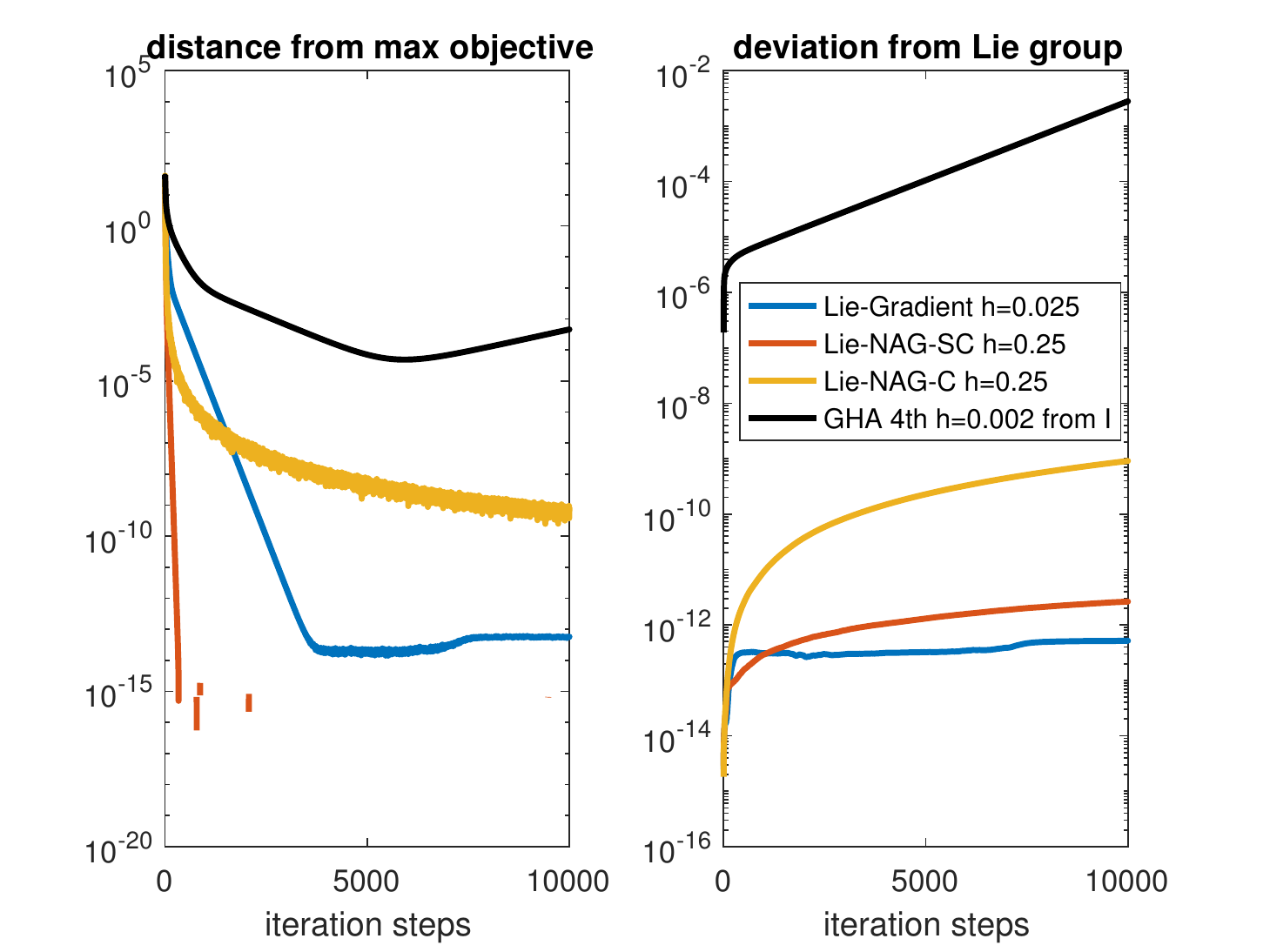}
\hspace{-20pt}
\vspace{-4ex}
\caption{Proposed Lie-GD, Lie-NAG-C and Lie-NAG-SC, compared with GHA, for computing the leading $l=2$ eigenvalues of $A=-\Xi \Xi^T/2$. $\Xi$ is 25-dimensional. Other descriptions are same as in Fig.\ref{fig_largestKevalues_differentA_1b_n500_new}.}
\label{fig_largestKevalues_differentA_4a_n25}
\vspace{-3.5ex}
\end{figure}

Fig.\ref{fig_largestKevalues_differentA_4a_n25} shows the advantage of variational methods (i.e., with momentum), even when the dimension is relatively low $n=25$. $A$ is defined such that its spectrum grows linearly with $n$, and GHA thus needs to use tiny timesteps. Although the proposed methods also need to use reduced step sizes for bigger $n$, the rate of reduction is much slower than that for GHA (results omitted).

\subsection{Stochastic Leading Eigenvalue Problems}
\label{sec_SG_experiments}
To investigate the efficacy of the proposed methods in the stochastic setup (Sec.\ref{sec_method_SG}), we take the same $A$ from Sec.\ref{sec_EV_bounded}, and add $K=100$ random perturbations to it to form a batch $A_1,\cdots,A_K$. Each random perturbation is $(\Xi+\Xi^T)/4/\sqrt{n}$ for i.i.d. $\Xi$; note these are large fluctuations when compared to $A$. Then $A$ is refreshed to be the mean of $A_k$'s, whose leading $l=2$ eigenvalues are accurately computed as the ground truth.

\begin{figure}[ht]
\vspace{-3ex}
\hspace{-20pt}
\includegraphics[width=1.16\columnwidth]{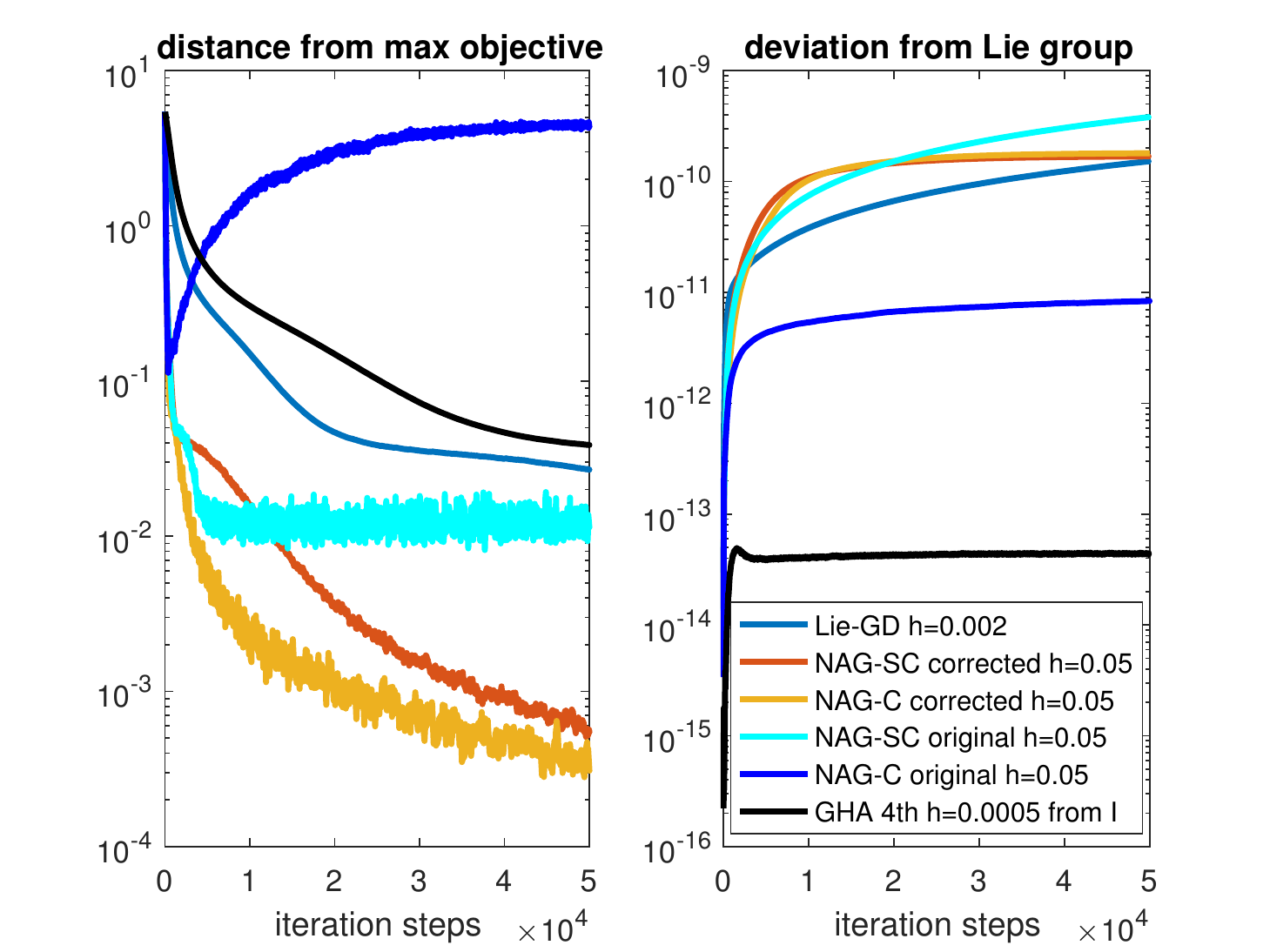}
\hspace{-20pt}
\vspace{-4ex}
\caption{The computation of leading $l=2$ eigenvalues of $A=\frac{1}{K}\sum_{k=1}^K A_k$ based on stochastic gradients from batch $A_1,\cdots,A_K$ without $A$. NAG-SC corrected, NAG-C corrected, NAG-SC and NAG-C use, respectively, $\gamma=1+0.01t$, $3/t+0.01t$, $1$, and $3/t$. Other descriptions are same as in Fig.\ref{fig_largestKevalues_differentA_1b_n500_new}.}
\label{fig_largestKevalues_differentAstochastic_n500new}
\vspace{-2ex}
\end{figure}

Fig.\ref{fig_largestKevalues_differentAstochastic_n500new} shows the advantage of variational methods, even though their larger step sizes lead to much higher variances of the stochastic gradient approximation. The corrected dissipation \eqref{eq_correctedDissipationSchedule} enabled the convergence of NAG-C. The same correction slows down the convergence of NAG-SC in the beginning, but significantly improves its long time performance, which otherwise stagnates at small but not infinitesimal error.

The reason NAG-SC-original stagnates is, over long time, it samples from an invariant distribution at a nonzero temperature. This invariant distribution, however, is not the exact one of the continuous limit; the latter of which would concentrate around the minimizer with 0 error. Instead, the numerical method's invariant distribution, if existent, is $\mathcal{O}(h^p)$ away from the exact one \citep{bou2010long, abdulle2014high} under suitable assumptions, which means as the numerical method converges, it gives $R$'s that are $\mathcal{O}(h^p)$ away from the exact minimizer with high probability. NAG-SC-corrected alleviated this issue.

\subsection{Leading Generalized Eigenvalue: a Demonstration Based on LDA}
\label{sec_LDA}
We report numerical experiments on multiclass Fisher Linear Discriminant Analysis (LDA) of the hand-written-digits database MNIST \citep{lecun1998gradient}. Since it is known that LDA can be formulated as a leading generalized eigenvalue problem (e.g., reviewed in  \cite{li2006using,welling2005fisher}; see appendix for a summary), we use it as an example to test our leading GEV algorithm. Important to note is, our purpose is NOT to construct an algorithm for MNIST classification, as it is known that LDA does not achieve state-of-art performance in that regard (test error based on exact leading GEV solution was $~10\%$ in our experiment). Instead, we simply would like to quantify the efficacy of our algorithm applied to a leading generalized eigenvalue problem based on real life data.

The 60000 training data of MNIST were employed to compute the `inter-class scatter matrix' A and the `intra-class scatter matrix' B (see appendix for more details). Each 28$\times$28 image had its white margins cropped, resulting in a 400-dimensional vector, and thus $A$ and $B$ are both 400-by-400, respectively positive semi-definite and positive definite. Furthermore, to avoid laborious tuning of timestep sizes, both $A$ and $B$ are normalized by their respective 2-norm; this is without loss of generality, because $\arg\min_Q \frac{\det(Q^T A Q)}{\det(Q^T B Q)}$ is invariant to scaling of $A$ and/or $B$. Since there are 10 classes, $l=9$ is chosen.

Note this is a positive semi-definite problem by construction. Some generalized eigenvalue methods require or prefer such a property (e.g., Oja flow \citep{yan1994global}), but the proposed algorithms are indifferent to the definiteness (see Rmk.\ref{rmk_shiftInvariance}).

\begin{figure}[ht]
\vspace{-2.5ex}
\hspace{-20pt}
\includegraphics[width=1.16\columnwidth]{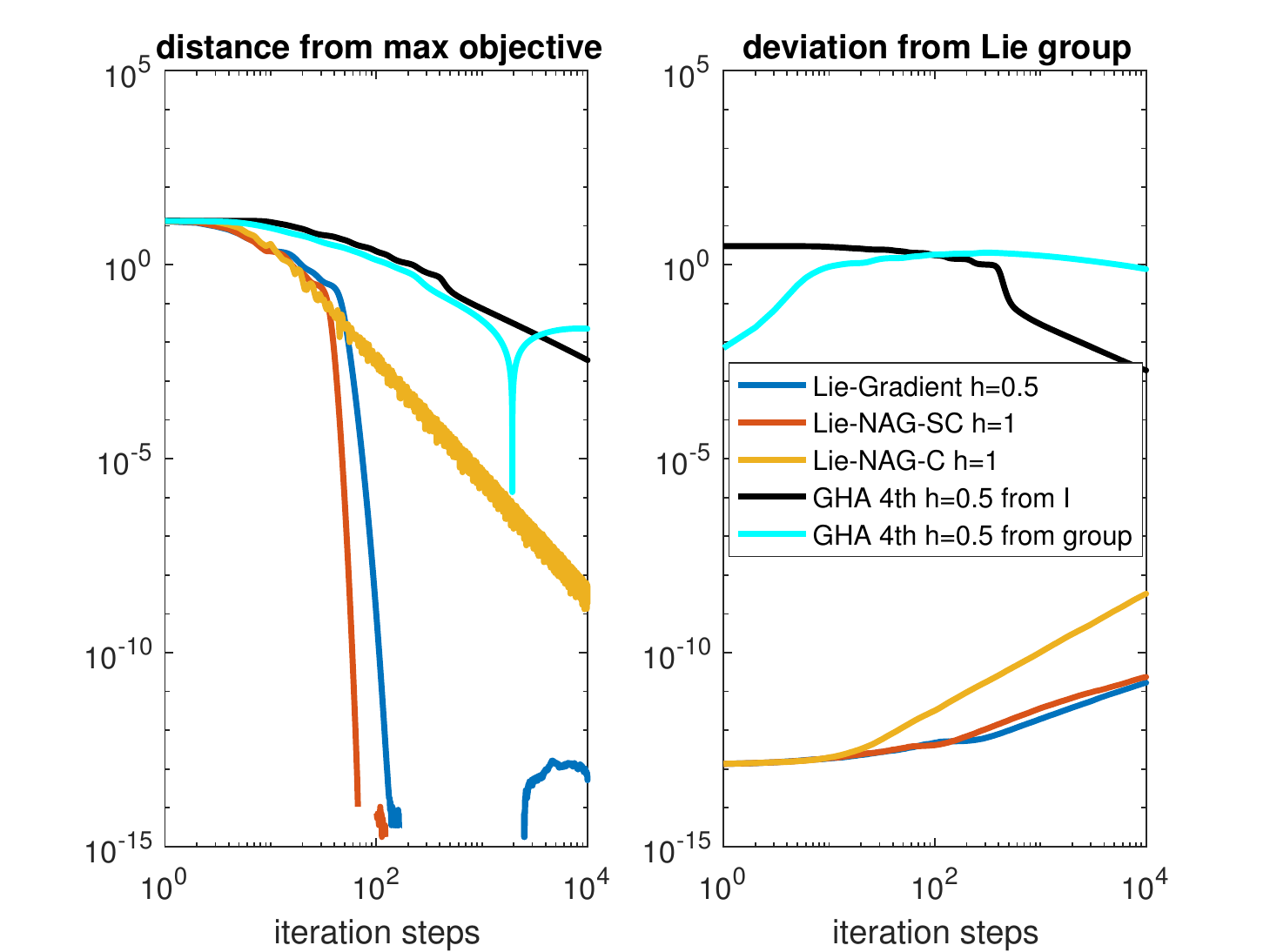}
\hspace{-20pt}
\vspace{-4ex}
\caption{Lie-GD, Lie-NAG-C, Lie-NAG-SC, and GHA, for computing the leading $l$ generalized eigenvalues associated with LDA for the MNIST dataset. All algorithms use step sizes tuned to minimize error in $10^4$ iterations (although the proposed methods do not need much tuning). Two GHA runs use two initial conditions, $Q(0)=I$ which is not on the Lie group $Q^T B Q=I$, and $Q(0)$ being the first $l$ columns of $L^{-1}$ which is on the Lie group; all others use initial condition $L^{-1}$. GHA was based on Runge-Kutta-4 integration of $\dot{Q}=(I-BQQ^T)A Q$ for accuracy, and an Euler integration did not result in any notable error reduction. NAG-SC uses friction coefficient untuned $\gamma=1$. The pollution of NAG simulations near the end is a machine precision artifact, and so are the deviations of Lie-NAGs and Lie-GD from the Lie group.}
\label{fig_largestKgeneralizedEvalues_MNIST_k9}
\vspace{-2ex}
\end{figure}

Fig.\ref{fig_largestKgeneralizedEvalues_MNIST_k9} shows that all proposed methods converge significantly faster than GHA. Interestingly, although Lie-NAG-SC still converges faster than Lie-GD, the acceleration due to momentum is not as drastic as before.

In addition, Fig.\ref{fig_largestKgeneralizedEvalues_MNIST_k9_noEigGap} in Appendix shows that our methods do not require an eigengap, and thus are widely applicable. Great methods have been continuously proposed for GEV; for instance, a globally linear convergent algorithm was recently proposed based on power method \citep{ge2016efficient}, but its convergence is affected by eigengap. The proposed methods do not have this restriction.

\section*{Acknowledgements}
The authors thank Tuo Zhao and Justin Romberg for insightful discussions. Generous support from NSF DMS-1521667, DMS-1847802 and ECCS-1936776 (MT) and CMMI-1824798 (TO) are acknowledged.

\clearpage

\bibliography{LieGroupOptimization}

\clearpage

\section*{Appendix}

\subsection*{Justification of NAG dynamics for GEV}
This section justifies why one can simply use the same NAG flow for eigenvalue problem and only modify $R$'s initial condition. It is rigorous when $B$ is positive definite, since its Cholesky decomposition will be used; otherwise, the justification is formal, and the same NAG dynamics is still well defined.

First, rewrite \eqref{eq_GEV} as
\begin{align*}
	\max_{R\in \mathbb{R}^{n\times n}} & \text{tr}(E^T R^T A R E) \nonumber\\
	\text{s.t.}\quad & R^T B R = I_{n\times n} .
\end{align*}
Cholesky decompose $B$ as $B=L^T L$, let $Q=LR$ and $\hat{A}=L^{-T} A L$, then the GEV is equivalently
\begin{align*}
	\max_{Q\in \mathbb{R}^{n\times n}} & \text{tr}(Q^T \hat{A} Q \mathcal{E}) \nonumber\\
	\text{s.t.}\quad & Q^T Q = I_{n\times n} .
\end{align*}
One can write down the NAG dynamics for variationally optimizing this problem:
\begin{equation*}
  \dot{Q} = Q \xi,
  \qquad
    \dot{\xi}
  = -\gamma(t) \xi
    + [Q^T \hat{A} Q, \mathcal{E}]
\end{equation*}
Note this is
\begin{equation*}
  L\dot{R} = L R \xi,
  \qquad
    \dot{\xi}
  = -\gamma \xi
    + [R^T L^T L^{-T} A L^{-1} L R, \mathcal{E}],
\end{equation*}
and all $L$'s can be canceled, leading to \eqref{eq_LieNAG_conti}.

In terms of initial condition, since $Q(0)^T Q(0)=I$, $R(0)^T L^T L R(0) = R(0)^T B R(0) = I$. $\xi(0)$ needs to be skew-symmetric throughout.

\subsection*{Preservation of Lie group structure}
(This section explicitly demonstrates several facts of geometric mechanics; for more information about geometric mechanics less in coordinates, see e.g., \cite{marsden2013introduction,hoscst2009}.)

For continuous dynamics, we have

\begin{Theorem}
	Consider $\dot{R}(t)=R(t) F(t)$ where $R$ and $F$ are $n$-by-$n$ matrices. If $R(t_0)^T B R(t_0)=I$ and $F(t)$ is skew-symmetric for all $t\geq t_0$, then $R(t)^T B R(t)=I$, $\forall t\geq t_0$.
\end{Theorem}

\begin{proof}
	\begin{align*}
		& ~\frac{d}{dt} (R^T B R) = \dot{R}^T B R + R^T B \dot{R} \\
		&= F^T R^T B R + R^T B R F = F^T + F = 0.\qedhere
	\end{align*}
\end{proof}

\begin{Corollary}
	We thus have Theorem \ref{eq_NAGstaysOnLieGroup}.
\end{Corollary}

\begin{proof}
	We only need to show $F:=\xi(t)$ remains skew-symmetric. This is true because
\begin{equation*}
\xi(t) = e^{-\Gamma(t)} \left( \xi(0) + \int_0^t e^{\Gamma(s)} [R(s)^T A R(s), \mathcal{E}] ds \right),
\end{equation*}
where $\Gamma(t) \defeq \int_0^t \gamma(s) ds$ is a scalar.
However, $\xi(0)$ is skew-symmetric by assumption, and so is the integrand because
\begin{gather*}
  [R(s)^T A R(s),\mathcal{E}]^T = [\mathcal{E}^T, (R(s)^T A R(s))^T] \\
  = [\mathcal{E}, R(s)^T A R(s)] = -[R(s)^T A R(s),\mathcal{E}]. \qedhere
\end{gather*}
\end{proof}

\begin{Corollary}
	Lie-GD $\dot{R}=R[R^T A R,\mathcal{E}]$ also maintains $R^T B R=I$.
\end{Corollary}

For discrete timesteppings, we have

\begin{Theorem}
	Define Cayley transformation as $\text{Cayley}(\xi):=(I-\xi/2)^{-1} (I+\xi/2)$.
	Consider $\dot{R}(t)=R(t)F(t)$ where $R$ and $F$ are $n$-by-$n$ matrices. If $R(t_0)^T B R(t_0)=I$ and $F(t_0)$ is skew-symmetric, then the discrete updates given by $\hat{R}=R(t_0) \exp(F(t_0) h)$ and $\hat{R}=R(t_0) \text{Cayley}(F(t_0) h)$ both satisfy $\hat{R}^T B \hat{R}=I$. 
	\label{thm_LieGroupPreservation_discrete}
\end{Theorem}

\begin{proof}
	Consider $\hat{R}=R Q$. If $Q^T Q=I$, then
	\[
		\hat{R}^T B \hat{R} = Q^T R^T B R Q = Q^T Q = I.
	\]
	$Q=\exp(F h)$ for skew-symmetric $F$ satisfies this condition because
	\[
		Q^T Q = \exp(F^T h) \exp(F h) = \exp(-F h)\exp(Fh) = I.
	\]
	$Q=\text{Cayley}(F h)$ for skew-symmetric $F$ satisfies this condition because
	{\small
	\begin{align*}
		&Q^T Q = (I+F h/2)^T (I-F h/2)^{-T} (I-F h/2)^{-1} (I+F h/2) \\
		&= (I-F h/2) (I+Fh/2)^{-1} (I-F h/2)^{-1} (I+F h/2) = I
	\end{align*}
	}
	the last equality because $I-Fh/2$ and $I+Fh/2$ commute.
\end{proof}

\subsection*{A brief recap of GHA}
(This subsection is not new research but for the self-containment of the article.)

Oja flow / Sanger's rule / Generalized Hebbian Algorithm (e.g., \cite{oja_simplified_1982, sanger_optimal_1989, gorrell_generalized_nodate, wei-yong_yan_global_1994}) is a celebrated type of methods based on continuous dynamics for finding leading eigenvalues of a symmetric matrix. Only for the reason of a concise presentation, we refer to them as GHA in this article.

GHA works as follows: given $n$-by-$n$ symmetric $A$, to find the eigenspace associated with its largest $l$ eigenvalues, one denotes by $V(t)$ an $n$-by-$l$ matrix and uses the long time limit of dynamics
\[
	\dot{V}=(I-V V^T) AV
\]
as a span of the corresponding orthonormal eigenvectors.

This approach can be extended to GEV \eqref{eq_GEV} by using GHA dynamics
\begin{equation}
	\dot{V}=(I-B V V^T) AV ;
	\label{eq_GHAdynamics}
\end{equation}
see e.g., \cite{chen_constrained_2019} and references therein.

To implement GHA in practice, the continuous dynamics need to be  numerically discretized. A 1st-order discretization is based on Euler scheme, namely
\[
	V_{i+1}=V_i+h(I-B V_i V_i^T) AV_i,
\]
and it is most commonly used. However, if a smaller deviation from the continuous dynamics is desired, a higher-order discretization can also be used, e.g., a 4th-order Runge-Kutta given by
{\small
\begin{align*}
	k_1 &= \left(I-B V_i V_i^T\right) AV_i \\
	k_2 &= \left(I-B \left(V_i+\frac{h}{2} k_1\right) \left(V_i+\frac{h}{2} k_1\right)^T\right) A \left(V_i+\frac{h}{2} k_1\right) \\
	k_3 &= \left(I-B \left(V_i+\frac{h}{2} k_2\right) \left(V_i+\frac{h}{2} k_2\right)^T\right) A \left(V_i+\frac{h}{2} k_2\right) \\
	k_4 &= \left(I-B \left(V_i+h k_3\right) \left(V_i+h k_3\right)^T\right) A \left(V_i+h k_3\right) \\
	V_{i+1} &= V_i + \frac{h}{6} \left(k_1+2k_2+2k_3+k_4\right) .
\end{align*}
}
Roughly 4 times the flops of Euler are needed per step, but the deviation from \eqref{eq_GHAdynamics} is $\mathcal{O}(h^4)$ instead of $\mathcal{O}(h)$ for Euler.

\subsection*{A brief recap of multiclass Fisher Linear Discriminant Analysis (LDA)}
(This subsection is not new research but, for the self-containment of the article, a quick excerpt of the existing methods of Fisher Linear Discriminant Analysis \cite{fisher1936use} and Multiple Discriminant Analysis (e.g., \cite{johnson2002applied}), mainly based on \cite{li2006using}).

Given $d$-by-1 vectorial data $x_i$, $i=1,\cdots,N$ labeled into $M$-classes, define `inter-class scatter matrix' $A$ and `intra-class class scatter matrix' $B$ by
\begin{align*}
	&\mu_m = \frac{1}{|\mathcal{C}_m|} \sum_{i\in \mathcal{C}_m} x_i, \\		&\bar{x} = \frac{1}{N} \sum_{i=1}^N x_i, \\
	& A = \sum_{m=1}^{M} (\mu_m - \bar{x}) (\mu_m - \bar{x})^T, \\
	& B = \sum_{m=1}^{M} \sum_{i\in \mathcal{C}_m} (x_i-\mu_m) (x_i-\mu_m)^T,
\end{align*}
where $\mathcal{C}_m$ is the set of indices corresponding to class-$m$. FDA seeks a projection represented by a $d$-by-$l$ matrix $Q$ that maximizes the Rayleigh quotient:
\[
	\max_Q \frac{\det(Q^T A Q)}{\det(Q^T B Q)},
\]
where a standard choice of $l$ is $l=M-1$. This problem can be reformulated as the generalized eigenvalue problem $A w=\lambda B w$ (e.g., \cite{li2006using,welling2005fisher}), and thus equivalent to
\begin{align*}
	&\max && \text{tr } (Q^T A Q) \\
	&\text{ s.t.} && Q^T B Q=I.
\end{align*}

\subsection*{Additional LDA experimental results}
To demonstrate that the proposed methods still work when there is no eigengap (i.e., two largest eigenvalues being identical), we take $A$ and $B$ from LDA for MNIST, Cholesky decompose $B$ as $B=L^T L$, let $\hat{A}=L^{-T} A L^{-1}$, diagonalize $\hat{A}=VDV^{-1}$, and then replace $D$'s largest diagonal element by the value of the 2nd largest. Denoting the result by $\tilde{D}$, we replace $A$ by $\tilde{A}=L^T V\tilde{D}V^{-1} L$. The generalized eigenvalue problem associated with $\{\tilde{A},B\}$ now has a zero eigengap, which prevents, for example, power-method based approaches from working. However, Fig. \ref{fig_largestKgeneralizedEvalues_MNIST_k9_noEigGap} shows that the proposed methods perform almost identically to the original $\{A,B\}$ case (c.f., Fig. \ref{fig_largestKgeneralizedEvalues_MNIST_k9}).

\begin{figure}[ht]
\hspace{-20pt}
\includegraphics[width=1.16\columnwidth]{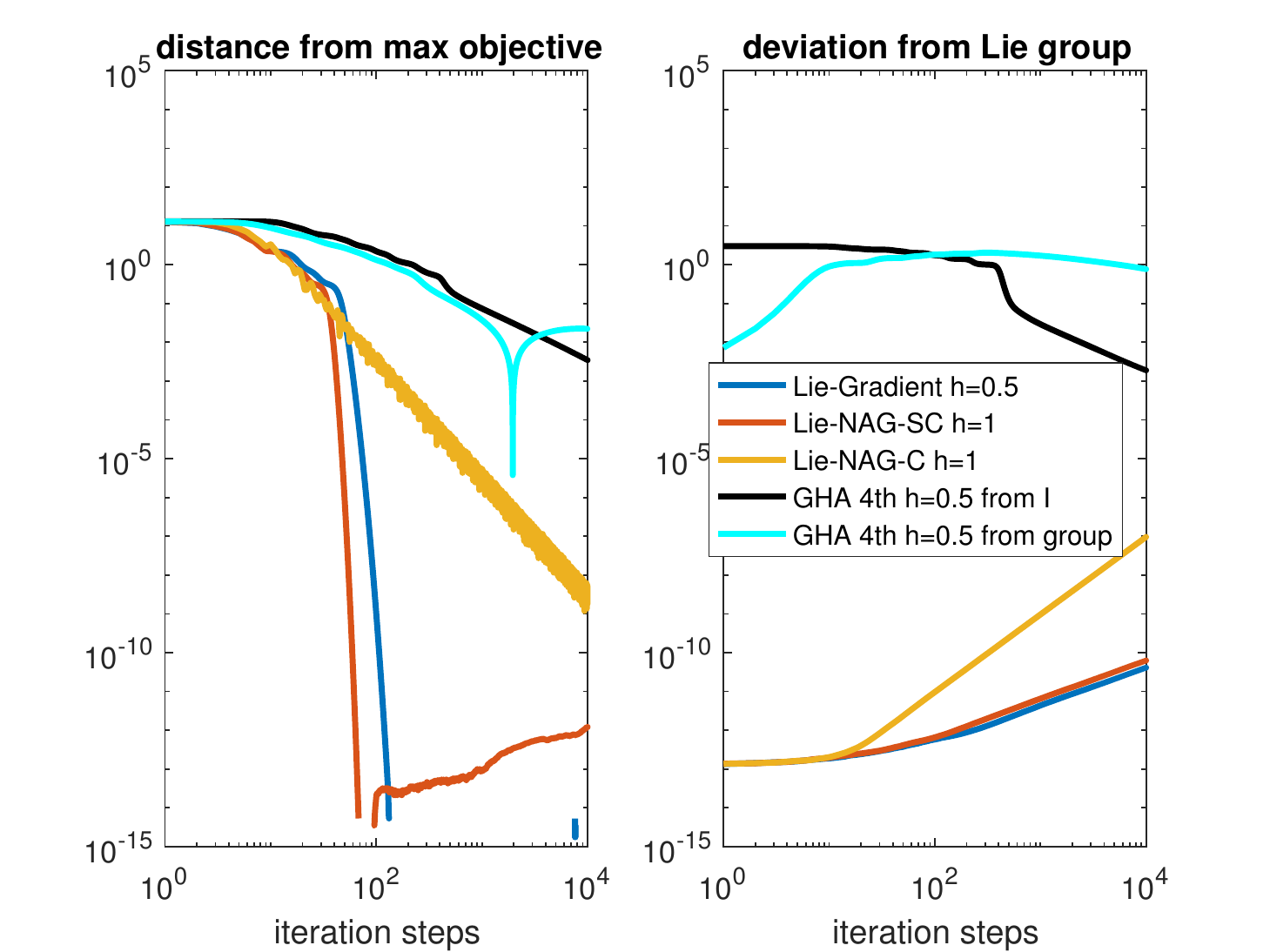}
\hspace{-20pt}
\caption{Same experiment as in Fig.\ref{fig_largestKgeneralizedEvalues_MNIST_k9} for modified MNIST with 0 eigengap.}
\label{fig_largestKgeneralizedEvalues_MNIST_k9_noEigGap}
\end{figure}

\subsection*{$l$ largest eigenvalues of $A=(\Xi+\Xi^T)/2/\sqrt{n}$: $n=2000$ result}
\begin{figure}[ht]
\hspace{-20pt}
\includegraphics[width=1.16\columnwidth]{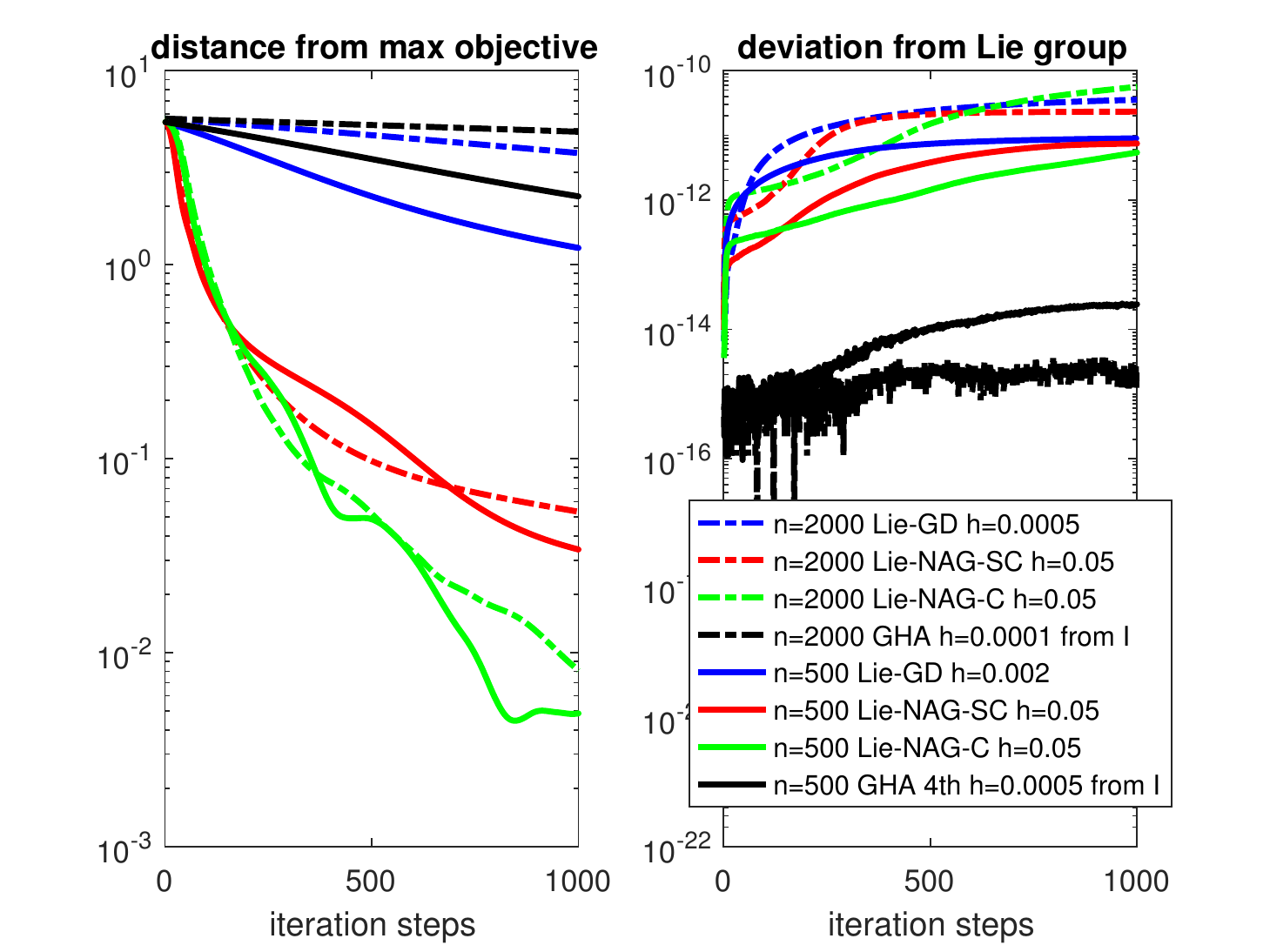}
\hspace{-20pt}
\caption{The computation of leading $l=2$ eigenvalues of 2000-dimensional scaled GOE, compared with that for 500-dimension. Other descriptions are same as in Fig.\ref{fig_largestKevalues_differentA_1b_n500_new}.}
\label{fig_largestKevalues_differentA_1b_n2000}
\end{figure}
Fig.\ref{fig_largestKevalues_differentA_1b_n2000} describes the same experiment as in Sec.\ref{sec_EV_bounded} when the dimension is $n=2000$ instead of $500$. When compared with the $n=500$ case, one sees Lie-GD and GHA converge much slower, but Lie-NAG's converge only marginally slower. This suggests that the advantage of variational methods increases in higher dimension, at least in this experiment.

\subsection*{$l$ largest eigenvalues of $A=(\Xi+\Xi^T)/2/\sqrt{n}$: $n=500$ result in wallclock count}
\begin{figure}[ht]
\hspace{-20pt}
\includegraphics[width=1.16\columnwidth]{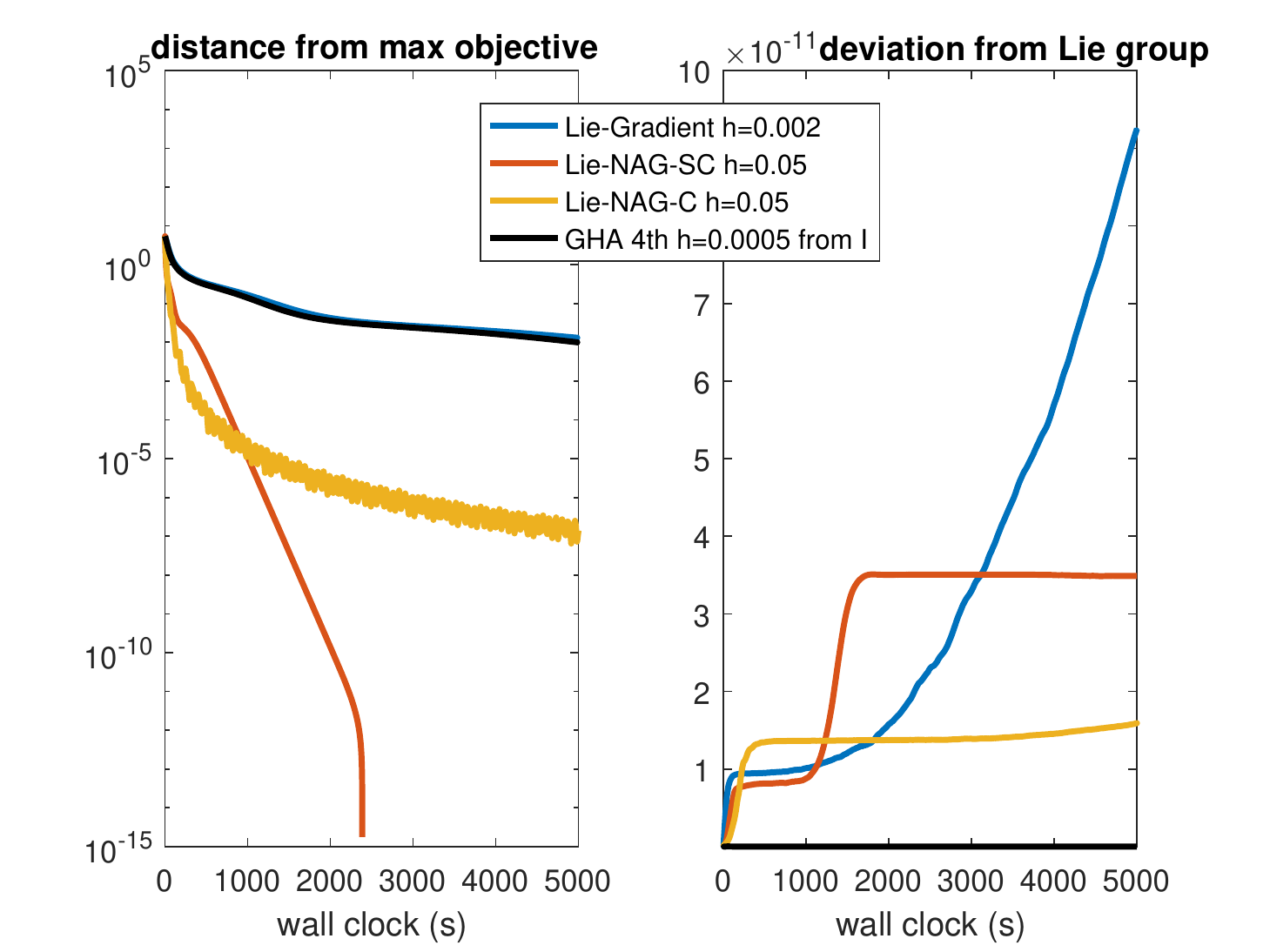}
\hspace{-20pt}
\caption{The computation of leading $l=2$ eigenvalues of 500-dimensional scaled GOE. All descriptions are same as in Fig.\ref{fig_largestKevalues_differentA_1b_n500_new}, except that x-axis is no longer in iteration steps but in wallclock.}
\label{fig_largestKevalues_differentA_1b_n500_new_wallClock}
\end{figure}
Fig.\ref{fig_largestKevalues_differentA_1b_n500_new_wallClock} illustrates the actual computational costs of methods used in this paper by reproducing Fig.\ref{fig_largestKevalues_differentA_1b_n500_new} with x-axis replaced by the time it took for each method to run. All qualitative conclusions remain unchanged. Experiments were conducted on a 4th-gen Intel Core laptop with integrated graphics unit running 64-bit Windows 7 and MATLAB R2016b.

\subsection*{Two 4th-order versions of Lie-NAG algorithms}
Version 1: more accurate but more computation
\begin{align*}
&\quad \phi^h = \phi_2^{a_1 h} \circ \phi_1^{b_1 h} \circ \phi_2^{a_2 h} \circ \phi_1^{b_2 h} \circ \phi_2^{a_3 h} \circ \phi_1^{b_3 h} \\
& \circ \phi_2^{a_4 h} \circ \phi_1^{b_3 h} \circ \phi_2^{a_3 h} \circ \phi_1^{b_2 h} \circ \phi_2^{a_2 h} \circ \phi_1^{b_1 h} \circ \phi_2^{a_1 h}+\mathcal{O}(h^5)
\end{align*}
where
\[
\begin{bmatrix} a_1 \\ a_2 \\ a_3 \\ a_4 \end{bmatrix} = 
\begin{bmatrix} 0.079203696431196 \\ 0.353172906049774 \\ -0.042065080357719 \\ 0.219376955753500 \end{bmatrix},
\]
\[
\begin{bmatrix} b_1 \\ b_2 \\ b_3 \end{bmatrix} = 
\begin{bmatrix} 0.209515106613362 \\ -0.143851773179818 \\ 0.434336666566456 \end{bmatrix}.
\]
Version 2: less accurate but less computation
\[
\phi_2^{a_1 h} \circ \phi_1^{b_1 h} \circ \phi_2^{a_2 h} \circ \phi_1^{b_2 h} \circ \phi_2^{a_2 h} \circ \phi_1^{b_1 h} \circ \phi_2^{a_1 h}
\]
where
\[
\begin{bmatrix} a_1 \\ a_2 \end{bmatrix} = 
\begin{bmatrix} \gamma_4/2 \\ (1-\gamma4)/2 \end{bmatrix},
\begin{bmatrix} b_1 \\ b_2 \end{bmatrix} = 
\begin{bmatrix} \gamma_4 \\ 1-2\gamma_4 \end{bmatrix}, \gamma_4=\frac{1}{2-2^{1/3}}.
\]
Details can be found, e.g., in \cite{mclachlan2002splitting}.
Swapping $\phi_1$ and $\phi_2$ will yield additional methods at the same order of accuracy. We present the above because $\phi_1$ is computationally more costly due to Cayley transform.

\subsection*{Some heuristic insights on the correction of the NAG dissipation coefficient in SG context}
Based on the discussion in the main text, heuristically, large $\gamma$ values correspond to lower `temperatures' and reduced variances accumulated from stochastic gradients. However, they also slow down the convergences of the stochastic processes, and yet we'd like to take advantage of the fast convergence of deterministic NAG dynamics. Therefore, we consider an additive correction that is small for small $t$ and increasing to infinity. 

For simplicity, restrict the correction to be a monomial of $t$, i.e., $\delta \gamma = c t^p$.
Then we select the value of $p$ by resorting to intuitions first gained from a linear deterministic case, for which our choice of $p$ has to lead to convergence because the deterministic solution is the mean of the stochastic solution. It is proved in \cite{artstein1976asymptotic} that a sufficient condition for asymptotic stability of $\ddot{q}+\gamma(t)\dot{q}+q=0$ is
\[
	\limsup_{T\rightarrow\infty} \left(\frac{1}{T^2}\int_0^T \gamma(t) dt \right) < \infty	\quad \text{and}\quad \gamma(t)\geq \gamma_0
\]
for some constant $\gamma_0>0$. It is easy to check that $\gamma(t)=\gamma_0+c t^p$ or $3/t+c t^p$ satisfies this condition if $p\leq 1$, but not when $p>1$. We thus inspect the boundary case of $p=1$ for a fast decay of variance at large $t$, now in a stochastic setup:
\begin{equation}
\begin{cases}
	dq &= p dt \\
	dp &= (-(\gamma_0(t)+c t)p - q) dt + \sigma dW
\end{cases}, 
\label{eq_toyLinearSDE}
\end{equation}
where $\gamma_0$ is either a constant or $3/t$. Since this is a linear SDE whose solution is Gaussian, it suffices to show the convergences of the (deterministic) mean and covariance evolutions in order to establish the SDE's convergence.

It is standard to show the mean $x(t):=\mathbb{E}[q(t),p(t)]$ satisfies a closed non-autonomous ODE system, and the covariance $V(t):=\mathbb{E}\big[ [ q(t)-\mathbb{E}[q(t)], ~ p(t)-\mathbb{E}[p(t)] ]^T [ q(t)-\mathbb{E}[q(t)], ~ p(t)-\mathbb{E}[p(t)] ] \big]$ satisfies another. These systems are not analytically solvable, but we can analyze their long time behavior by asymptotic analysis. 

More precisely, under the ansatz of $\mathbb{E}[q]=b t^a+o(t^a)$, matching leading order terms in the mean ODE leads to
\[
	\mathbb{E}[q(t)] \sim t^{-1/c},	\qquad \mathbb{E}[p(t)] \sim t^{-1/c-1}
\]
for both constant $\gamma_0$ and $\gamma_0(t)=3/t$ in \eqref{eq_toyLinearSDE}.

Under the ansatz of $\text{Var}[q]=b_1 t^{a_1}+o(t^{a_1})$, $\text{Var}[p]=b_2 t^{a_2}+o(t^{a_2})$, $\mathbb{E}[(q-\mathbb{E}q)(p-\mathbb{E}p)]=b_3 t^{a_3}+o(t^{a_3})$, matching leading order terms in the covariance ODE leads to
\begin{align*}
	\text{Var}[q] = \frac{1}{c(2-c)}t^{-1}, \quad \text{Var}[p] = \frac{1}{2c}t^{-1}, \\
	\mathbb{E}[(q-\mathbb{E}q)(p-\mathbb{E}p)] = \frac{1}{2c(c-2)}t^{-2}.	\quad
\end{align*}
Note this means, for small but positive $c$, convergence is guaranteed, and covariance converges slower than mean, at the rate independent of $c$.

Therefore, adding $ct$ to $\gamma$ in the original NAG's works in the linear case, and thus it has a potential to work for nonlinear cases (e.g., Lie group versions). And it does in experiments (Sec.\ref{sec_SG_experiments}).

\subsection*{Hamiltonian Formulation}
In this section, we give a Hamiltonian formulation of the variational optimization equation~\eqref{eq:E-L} and prove the conformal symplecticity of its flow.
\paragraph{Symplectic Structure on $\mathsf{G} \times \mathfrak{g}^{*}$}~\\ 
Let $\lambda$ be the left trivialization of $T^{*}\mathsf{G}$, i.e.,
\begin{equation*}
  \lambda \colon T^{*}\mathsf{G} \to \mathsf{G} \times \mathfrak{g}^{*};
  \qquad
  p_{g} \mapsto \parentheses{ g, T_{e}^{*}L_{g}(p_{g}) }.
\end{equation*}
Then its inverse is given by
\begin{equation*}
  \lambda^{-1}\colon \mathsf{G} \times \mathfrak{g}^{*} \to T^{*}\mathsf{G};
  \qquad
  (g,\mu) \mapsto T_{g}^{*}L_{g^{-1}}(\mu).
\end{equation*}
Let $\Theta$ and $\Omega \defeq -\d\Theta$ be the canonical one-form and the symplectic structure on $T^{*}\mathsf{G}$, and $\theta$ and $\omega$ be their pull-backs via the left trivialization, i.e.,
\begin{equation*}
  \theta \defeq (\lambda^{-1})^{*} \Theta,
  \qquad
  \omega \defeq (\lambda^{-1})^{*} \Omega.
\end{equation*}
According to \citet[Proposition~4.4.1 on p.~315]{AbMa1978} (see also the reference therein), for any $(g,\mu) \in \mathsf{G} \times \mathfrak{g}^{*}$ and any $(v,\alpha), (w,\beta) \in T_{(g,\mu)}(\mathsf{G} \times \mathfrak{g}^{*})$,
\begin{equation}
  \label{eq:theta}
  \theta_{(g,\mu)}(w,\beta)
  = \ip{ \mu }{ T_{g}L_{g^{-1}}(w) }
\end{equation}
and
\begin{equation}
  \label{eq:omega}
  \begin{split}
  \omega_{(g,\mu)}&((v,\alpha),(w,\beta)) \\
  &= \ip{ \beta }{ T_{g}L_{g^{-1}}(v) } - \ip{ \alpha }{ T_{g}L_{g^{-1}}(w) } \\
  &\quad + \ip{ \mu }{ [T_{g}L_{g^{-1}}(v), T_{g}L_{g^{-1}}(w)] }.
  \end{split}
\end{equation}
Given a function $h\colon \mathsf{G} \times \mathfrak{g}^{*} \to \R$, the corresponding Hamiltonian vector field $X_{h} \in \mathfrak{X}(\mathsf{G} \times \mathfrak{g}^{*})$ defined by $\ins{X_{h}}{\omega} = \d{h}$ is given by
\begin{equation*}
  X_{h}(g,\mu) = \parentheses{
    T_{e}L_{g}\parentheses{ \fd{h}{\mu} },\,
    \ad_{\fd{h}{\mu}}^{*}\mu - T_{e}^{*}L_{g}(\d_{g}h)
  },
\end{equation*}
where $\d_{g}$ stands for the exterior differential with respect to $g$.


\paragraph{Legendre Transform and Hamiltonian Formulation}~\\
We may apply a time-independent Legendre transform using the initial Lagrangian as follows:
Let us define the initial Lagrangian $L_{0}\colon \mathsf{G} \times \mathfrak{g} \to \R$ by setting $L_{0}(g,\xi) \defeq L(g,\xi,0)$, and the time-independent Legendre transform
\begin{equation*}
  \FL_{0}\colon \mathfrak{g} \to \mathfrak{g}^{*};
  \qquad
  \xi \mapsto \fd{L_{0}}{\xi}(g,\xi,t) = r(0)\,\mathbb{I}(\xi),
\end{equation*}
whose inverse is given by
\begin{equation*}
  (\FL_{0})^{-1}\colon \mathfrak{g}^{*} \to \mathfrak{g};
  \qquad
  \mu \mapsto \frac{1}{r(0)}\mathbb{I}^{-1}(\mu).
\end{equation*}
We define the initial Hamiltonian $H\colon \mathsf{G} \times \mathfrak{g}^{*} \to \R$ as follows:
\begin{align*}
  H(g, \mu)
  &\defeq \ip{\mu}{(\FL_{0})^{-1}(\mu)} - L_{0}\parentheses{g, (\FL_{0})^{-1}(\mu) } \\
  &= \frac{1}{2r(0)} \ip{\mu}{\mathbb{I}^{-1}(\mu)} + r(0) f(g).
\end{align*}
Its associated Hamiltonian vector field $X_{H}$ on $\mathfrak{g}^{*}$ is defined as $\ins{X_{H}}\omega = \d{H}$ using the symplectic form $\omega$ on $\mathsf{G} \times \mathfrak{g}^{*}$ (see \eqref{eq:omega}):
\begin{equation*}
  X_{H}(\mu) = \ad_{\fd{H}{\mu}}^{*}\mu - T_{e}^{*}\mathsf{L}_{g}(\d_{g}H).
\end{equation*}
Then we may rewrite \eqref{eq:E-L} as follows:
\begin{equation}
  \label{eq:Ham0}
  \begin{split}
  \dot{\mu}
  &= -\gamma(t) \mu + \ad_{\fd{H}{\mu}}^{*}\mu - T_{e}^{*}\mathsf{L}_{g}(\d_{g}H) \\
  &= X_{H}(\mu) - \gamma(t) \mu,
  \end{split}
\end{equation}
where we set $\gamma(t) \defeq r'(t)/r(t)$.

\paragraph{Conformal Symplecticity}~\\
Given the Lagrangian of the form $r(t) L_{0}(q,\dot{q})$, the Euler--Lagrange equation is
\begin{equation}
  \label{eq:E-L_0}
  \od{}{t}\parentheses{ r(t)\pd{L_{0}}{\dot{q}} } - r(t)\pd{L_{0}}{q} = 0.
\end{equation}
We would like to show that the two-form $r(t) \d{p}\wedge \d{q}$ with $p \defeq \tpd{L_{0}}{\dot{q}}$ is preserved in time in two different ways.
The first is based on the variational principle: Consider
\[
    \d\d \int_{t_0}^{t_1} r(t) L_{0}(q,\dot{q}) dt,
\]
which is obviously 0 because any exact form is closed. On the other hand, it is the same as (due to integration by parts)
\[
    \d \left( \int_{t_0}^{t_1} \left( r \frac{\partial L_{0}}{\partial q} \d{q} - \frac{d}{dt}\left(r \frac{\partial L_{0}}{\partial \dot{q}} \right) \d{q} \right) dt + \left. r \frac{\partial L_{0}}{\partial \dot{q}} \d{q} \right|_{t_0}^{t_1} \right)
\]
The first term is zero because of \eqref{eq:E-L_0}. Therefore,
\[
  0=\d\left( \left. r \frac{\partial L_{0}}{\partial \dot{q}} \d{q} \right|_{t_0}^{t_1} \right)
  = \d (r p \d{q})|_{t_0}^{t_1} = r \d{p}\wedge \d{q} |_{t_0}^{t_1}
\]

The second proof uses the Hamiltonian formulation.
We may write the Hamiltonian system corresponding to the Euler--Lagrange equation for the Lagrangian of the form $r(t) L_{0}(q,\dot{q})$ as follows:
\begin{equation}
  \label{eq:Ham_with_dissipation}
  \dot{q} = \pd{H}{p},
  \qquad
  \dot{p} = -\pd{H}{q} - \gamma(t)p,
\end{equation}
where the Hamiltonian $H$ is obtained via the Legendre transform of $L_{0}(q,\dot{q})$ not $r(t) L_{0}(q,\dot{q})$.

In what follows, we would like to generalize the work of \citet{McPe2001}---in which $\gamma$ is set to be constant---to derive the conformal symplecticity of dissipative Hamiltonian systems of the above type.
Let $P$ be an (exact) symplectic manifold with symplectic form $\Omega = -\d\Theta$ and $H\colon P \to \R$ be a (time-independent) Hamiltonian.
Let us define a time-dependent vector field $X_{H,(\cdot)}\colon \R \times P \to TP$ by defining, for any $t \in \R$, a vector field $X_{H,t}$ on $P$ by setting
\begin{equation*}
  X_{H,t} \defeq X_{H} - Z_{t},
\end{equation*}
where $X_{H}$ is the Hamiltonian vector field on $P$ defined by
\begin{equation*}
  \ins{X_{H}}\Omega = \d{H},
\end{equation*}
and the time-dependent vector field $Z_{(\cdot)}\colon \R \times P \to TP$ is defined as follows:
Let $\Omega_{(\cdot)}$ be the time-dependent symplectic form on $P$ defined as, for any $t \in \R$, 
\begin{equation*}
  \Omega_{t} \defeq r(t) \Omega.
\end{equation*}
We define $Z_{t}$ by setting
\begin{equation*}
  \ins{Z_{t}}\Omega_{t} = -r'(t)\Theta.
\end{equation*}

In terms of the canonical coordinates $(q,p)$ for $P$, we have
\begin{equation*}
  Z_{t} = p_{i} \pd{}{p_{i}},
\end{equation*}
and hence we have
\begin{equation*}
  X_{H,t}(q,p) = \pd{H}{p_{i}} \pd{}{q^{i}} + \parentheses{ \pd{H}{q^{i}} + \gamma(t)p_{i} } \pd{}{p_{i}}.
\end{equation*}
Therefore, $X_{H,t}$ yields the dissipative Hamiltonian system~\eqref{eq:Ham_with_dissipation}.

Let $\Phi \colon \R \times \R \times P \to P$ be the time-dependent flow of $X_{H,(\cdot)}$ (assuming for simplicity that the solutions exist for any time $t \in \R$ with any initial time $t_{0} \in \R$).
Then, for any $t_{0}, t_{1} \in \R$ (see, e.g., \citet[Proposition~22.15]{Le2013}),
\begin{align*}
  &\left.\od{}{t} \Phi_{t,t_{0}}^{*} \Omega_{t}\right|_{t=t_{1}} \\
  &= \Phi_{t_{1},t_{0}}^{*} \parentheses{ \left.\pd{}{t}\Omega_{t}\right|_{t=t_{1}} + \mathcal{L}_{X_{H,t_{1}}}\Omega_{t_{1}} } \\
  &= \Phi_{t_{1},t_{0}}^{*} \parentheses{ r'(t_{1})\Omega + \mathcal{L}_{X_{H}}\Omega_{t_{1}} + \mathcal{L}_{Z^{t_{1}}}\Omega_{t_{1}} } \\
  &= \Phi_{t_{1},t_{0}}^{*} \parentheses{ r'(t_{1})\Omega + r(t_{1})\mathcal{L}_{X_{H}}\Omega + r(t_{1})\mathcal{L}_{Z_{t_{1}}}\Omega } \\
  &= \Phi_{t_{1},t_{0}}^{*} \parentheses{ r'(t_{1})\Omega - r(t_{1})\parentheses{ \d\ins{Z_{t_{1}}}\Omega + \ins{Z_{t_{1}}}\d\Omega } } \\
  &= \Phi_{t_{1},t_{0}}^{*} \parentheses{ r'(t_{1})\Omega - \d\ins{Z_{t_{1}}}\Omega_{t_{1}} } \\
  &= \Phi_{t_{1},t_{0}}^{*} \parentheses{ r'(t_{1})\Omega - \d(-r'(t_{1})\Theta) } \\
  &= \Phi_{t_{1},t_{0}}^{*} \parentheses{ r'(t_{1})\Omega + r'(t_{1})\d\Theta } \\
  &= 0.
\end{align*}
Therefore, we have
\begin{equation}
  \label{eq:conformal_symplecticity}
  \Phi_{t_{1},t_{0}}^{*}\Omega_{t_{1}} = \Omega_{t_{0}}.
\end{equation}

Now, \eqref{eq:Ham0} is a special case of the above setting.
Specifically, we may define a time-dependent vector field $Z_{(\cdot)} \colon \R \times (\mathsf{G} \times \mathfrak{g}^{*}) \to T(\mathsf{G} \times \mathfrak{g}^{*})$ by setting, for any $t \in \R$,
\begin{equation*}
  \ins{Z_{t}}\omega_{t} = -r'(t)\theta,
\end{equation*}
where $\omega_{t} \defeq r(t) \omega$.
This yields $Z_{t}(\mu) = \gamma(t) \mu$.
Then we may write \eqref{eq:Ham0} as 
\begin{equation*}
  \dot{\mu}(t) = (X_{H} - Z_{t})(\mu(t)).
\end{equation*}
Let $\varphi\colon \R \times \R \times (\mathsf{G} \times \mathfrak{g}^{*}) \to \mathsf{G} \times \mathfrak{g}^{*}$ be the time-dependent flow of this system.
Then, the conformal symplecticity~\eqref{eq:conformal_symplecticity} implies that, for any $t_{0}, t_{1} \in \R$,
\begin{equation*}
  \varphi_{t,t_{0}}^{*} \omega_{t} = \omega_{t_{0}}.
\end{equation*}

\end{document}